\documentclass{article}
\usepackage{rst_preprint,times}
\usepackage{amsmath,amssymb,amsfonts,amsthm,mathtools}
\usepackage{microtype}
\usepackage{booktabs}
\usepackage{array}
\usepackage{float}
\usepackage{placeins}
\usepackage{longtable}
\usepackage{enumitem}
\usepackage{hyperref}
\usepackage{url}
\usepackage{graphicx}
\usepackage{tikz}
\usetikzlibrary{arrows.meta,positioning,calc}
\usepackage{pgfplots}
\pgfplotsset{compat=1.18}
\hypersetup{colorlinks=true,linkcolor=blue,citecolor=blue,urlcolor=blue}
\setlist{topsep=2pt,itemsep=0pt,parsep=0pt,partopsep=0pt}

\title{RST: Certifying and Constructing Prescribed \mbox{Information} in Variational Autoencoders}
\author{%
Zegu Zhang$^{1}$ \quad Jianhua Peng$^{2}$ \quad Jian Zhang$^{3}$\\[-1mm]
{\small $^{1}$Independent Researcher}\\
{\small $^{2}$School of Computing, Southeast University}\\
{\small $^{3}$Independent Researcher}\\[1mm]
{\small \texttt{zeguzhang@outlook.com} \quad \texttt{2018100913@niit.edu.cn}}\\
{\small \texttt{zhangjian.1993@tsinghua.org.cn}}
}
\date{}

\newtheorem{theorem}{Theorem}
\newtheorem{proposition}{Proposition}
\newtheorem{corollary}{Corollary}

\newcommand{\E}{\mathbb{E}}
\newcommand{\R}{\mathbb{R}}
\newcommand{\N}{\mathcal N}
\newcommand{\KL}{\mathrm{KL}}
\newcommand{\softmax}{\operatorname{softmax}}
\newcommand{\LTS}{L_{\mathrm{TS}}}
\newcommand{\IT}{I_T}

\newcommand{\G}{G_T}
\newcommand{\uphi}{u_\phi}
\newcommand{\uT}{u_T}

\newcommand{\mainback}[2]{\noindent{\footnotesize\emph{\hyperref[#1]{Back to main text: #2}}}\par\vspace{-0.35em}}
\begin{document}
\maketitle
\begin{abstract}
Posterior-collapse diagnostics can establish that a VAE uses its latent representation, but they do not determine whether a specified information view crosses a representation/readout interface fixed before training.  Reconstruction--Student--Teacher (RST) makes that stronger requirement explicit: \textbf{R denotes Reconstruction}, S the fixed Student witness, and T the frozen Teacher specification.  Let $T_X=T(\cdot\mid X)$ be a full-support categorical Teacher, $C\mid X\sim T_X$, $\bar T=\mathbb E_XT_X$, and let an audited representation $Z$ satisfy $C-X-Z$.  With $I_T:=I(C;X)$, $L_{\rm TS}:=\mathbb E\,\KL(T_X\|S(\cdot\mid Z))$, and $G_T:=I_T-L_{\rm TS}$, we prove $I(C;Z)\ge G_T$ and $L_{\rm TS}\ge I_T$ for every input-independent $Z\perp X$, with equality attained by the centered RST witness.  Thus $G_T>0$ certifies transmission of the declared Teacher information through the specified interface.  In the main experiments $Z=u_\phi(X)$ is the deterministic encoder mean, while reconstruction uses $z\sim q_\phi(z\mid x)$.

A centered regular-simplex witness turns the certificate into an explicit construction: a minimum-dimensional closed-form Teacher code, its complete affine solution fiber, exact witness-visible/witness-null routing, a margin--energy path, strict function-space departure from the centered input-independent point, certified preservation cylinders, and orthogonal multi-view composition.  Five-dataset experiments test boundary crossing, stronger KL pressure, Teacher counterfactuals, witness-null routing, escape, and preservation.  On CIFAR-100, Aux codes are probe-readable while their predeclared margins are negative.  Cross-seed Aux-head reuse falls from $36.57/21.88\%$ native accuracy to $5.20/0.94\%$, whereas one fixed RST witness directly reads all five Conv-RST encoders at $43.84/33.58\%$.  The same witness is also certificate-positive on five independently optimized ResVAE RST encoders ($G_c=0.869\pm0.057$, $G_f=1.083\pm0.105$); Conv$\leftrightarrow$Res learned-head reuse is near chance, while a ridge-affine adapter makes all $100/100$ transferred margins positive.  These tests validate a functional consequence of prescribed alignment; they do not claim fixed simplex classifiers or compatible representations as new.
\end{abstract}
\vspace{-1.00em}
\section{Introduction}\label{sec:main-introduction}
A variational autoencoder (VAE)~\citep{kingma2014vae,rezende2014stochastic} can fail to use its representation in more than one way.  The posterior may become input-independent; the code may remain active while discarding a factor of interest; or the factor may be present but only in an arbitrary orientation that must be discovered by fitting a new probe.  These cases are not equivalent.  Ordinary downstream prediction often needs only the last property, but modular or auditable representations can require more: a route and its readout geometry may have to be declared before training and then reused, checked, or composed without learning a new coordinate system.  We call this requirement \emph{prescribed alignment}.

\noindent\textbf{Intuition.}  The frozen Teacher says \emph{what} should be readable; the fixed Student witness says \emph{how} it must be read; and the encoder must make that predeclared reader succeed.  The certificate asks whether this success is better than any input-independent code.  Figure~\ref{fig:main-geometry} then turns the same test into geometry: visible directions carry the declared view, witness-null directions leave residual freedom, and an analytic path moves from the centered reference toward an exact Teacher code.

Formally, RST names three functional relations: \textbf{R=Reconstruction}, the ordinary VAE reconstruction path; \textbf{S=Student}, the fixed witness/readout geometry; and \textbf{T=Teacher}, the frozen information specification.  Throughout, $Z$ denotes the audited representation (main experiments: $Z=Z_\mu:=u_\phi(X)$), while lowercase $z\sim q_\phi(z\mid x)$ is the sampled latent sent to the decoder.  With $T_X:=T(\cdot\mid X)$, $C\mid X\sim T_X$, $\bar T:=\mathbb E_XT_X$, and $C-X-Z$, define
\[
I_T=\mathbb E_X\KL(T_X\|\bar T),\qquad
L_{\mathrm{TS}}=\mathbb E_{X,Z}\KL(T_X\|S(\cdot\mid Z)),\qquad
G_T=I_T-L_{\mathrm{TS}}.
\]
Here $I_T$ measures declared Teacher information, $L_{\mathrm{TS}}$ is fixed-reader mismatch, and $G_T$ is the certificate margin.  The exact decomposition below gives $I(C;Z)\ge G_T$.  Every input-independent representation $Z\perp X$ satisfies $L_{\mathrm{TS}}\ge I_T$, and the centered simplex witness attains equality at the constant origin.  The corresponding best-constant/input-independent boundary was established in the earlier Teacher-guided certificate~\citep{zhang2026teacher}; here it is the reference point for the constructive geometry, not a new boundary claim.  Input-independent means $Z\perp X$, distinct from the later witness-null directions, which are invisible to the fixed witness but need not be input-independent.
Seen through this formulation, exact constant posterior collapse is the canonical extreme case: an input-independent representation cannot transmit any nontrivial declared Teacher view.  Hence $G_T>0$ rules out that collapse on the audited route, while the broader prescribed-information formulation also distinguishes failures that ordinary non-collapse diagnostics can miss.  Posterior collapse is therefore the central input-independent reference case, not the full scope of the theory.

The central question is what becomes possible once this input-independent reference boundary is fixed.  A boundary can tell us that a representation has crossed input independence, but it does not tell us how to realize a chosen Teacher, how much visible capacity that realization needs, what degrees of freedom remain available, how to depart from the input-independent reference and quantify the cost, whether that centered reference has a strict departure direction, how a positive-margin state can survive later optimization, or how several independently specified views can coexist.  For a full-support Teacher with $K$ classes, the regular-simplex construction answers these questions in one geometry:
\begin{center}
\resizebox{0.96\linewidth}{!}{$\boxed{\text{CERTIFY}\rightarrow\text{REALIZE}\rightarrow\text{ROUTE}\rightarrow\text{MOVE/COST}\rightarrow\text{ESCAPE}\rightarrow\text{PRESERVE}\rightarrow\text{COMPOSE}}$}
\end{center}
This sequence is the organizing dependency of the paper, not a list of unrelated claims.  The ingredients---KL decompositions, softmax log-odds, regular simplices, tight frames, and exponential interpolation---are classical; the contribution is the way they are assembled into a fixed prescribed-information interface with exact realizations, auditable witness-null directions, explicit energy budgets, and composable protected routes.  Under centered unit-row minimal-dimensional linear-softmax witnesses, requiring minimum visible dimension, an explicit certificate-invariant residual complement, and minimum normalized worst-case sensitivity selects the regular simplex uniquely up to rotation (Appendix~\ref{app:new-constructive-proofs}); the remaining construction then operates on this same geometry.

\begin{table}[t]
\centering
\footnotesize
\setlength{\tabcolsep}{3.4pt}
\caption{What the exact input-independent boundary leaves open, and what the simplex construction adds.}
\label{tab:novelty-map}
\begin{tabular}{>{\raggedright\arraybackslash}p{0.34\linewidth}>{\raggedright\arraybackslash}p{0.57\linewidth}}
\toprule
Question after certification & Construction developed here \\
\midrule
How can the declared Teacher be realized? & Closed-form minimum-norm Teacher code (denoted $u_T(x)$ below) in the minimal $K-1$ witness-visible dimensions \\
Where can residual variation live? & Complete affine solution fiber and exact certificate invariance along witness-null directions \\
How can one leave the input-independent reference, and at what cost? & Analytic interpolation to the Teacher code, with a certified margin, quadratic mean-prior energy, and a strict local departure direction \\
How can an informative state survive later training? & A visible-space preservation radius with unrestricted witness-null motion \\
How can several declared views coexist? & Orthogonal simplex blocks with simultaneous margins and additive visible-energy budgets \\
\bottomrule
\end{tabular}
\end{table}

The empirical program mirrors this dependency rather than optimizing a single benchmark score.  Five image datasets test boundary crossing; KL stress, Teacher counterfactuals, witness-null routing, escape, and certified refinement probe distinct parts of the theory.  CIFAR-100 then separates
\[
\boxed{\text{latent activity}\;\neq\;\text{learned decodability}\;\neq\;\text{prescribed alignment}.}
\]
The separation is both formal and functional.  Appendix~\ref{app:decodability-separation} gives an exact rotation counterexample.  In matched controls, frozen-encoder probes recover the Teacher while the predeclared Aux interfaces remain certificate-negative; learned Aux heads fail almost completely under both cross-seed and Conv$\leftrightarrow$Res no-refit reuse, but small adapters recover much of the lost performance.  In contrast, the same predeclared RST witness directly reads all five ConvVAE seeds and a second independently optimized ResVAE backbone.  RST therefore targets a narrower no-refit interface contract, not ordinary decodability or compatibility as a field-level novelty.

\vspace{-0.45em}
\noindent\textbf{Main contributions.}\par
\vspace{-0.15em}
\begin{enumerate}[leftmargin=1.55em]
\item \textbf{A route-local certificate written directly on the representation variable.}  For a frozen Teacher and a witness strictly positive on its support we give the exact information decomposition, the attainable input-independent boundary for the centered witness, and a re-instantiation principle for any chosen $Z_m$ satisfying the corresponding Markov condition.  Here $Z_m$ is a mathematical representation variable and is distinct from the R in RST, which denotes Reconstruction.
\item \textbf{A canonical prescribed-interface geometry.}  For a full-support $K$-class Teacher we derive the minimal $K-1$ witness-visible dimension, the unique minimum-norm Teacher code, the complete affine solution fiber, and exact visible/witness-null routing.  Under the stated normalized minimal-interface criterion, the simplex is unique up to rotation; simplices themselves are not claimed as new.
\item \textbf{Movement, cost, escape, and preservation in the same geometry.}  The analytic Teacher ray has a certified margin--energy law and a strictly negative alignment-loss derivative at the centered input-independent point, while a Lipschitz bound gives a positive-margin preservation cylinder around informative anchors.
\item \textbf{Composable prescribed views with functional interface controls.}  Orthogonal simplex blocks provide simultaneous route-wise guarantees and additive energy budgets.  CIFAR-100 recovers the predicted selective sign pattern in all five seeds; frozen probes show information can survive when the declared interface fails, and cross-seed/cross-architecture reader transfer localizes the failure to producer--consumer coordinates.  The same predeclared witness is reused without refitting across ConvVAE and ResVAE RST encoders; small adapters and controlled gauge/zero/swap interventions isolate the corresponding interface behavior.
\end{enumerate}

The claim boundary is intentionally explicit.  A positive $G_T$ is a sufficient certificate relative to the declared Teacher, route, witness, and evaluation distribution.  It does not prove semantic completeness, decoder use of that route, disentanglement, parameter-space convergence, or the absence of every other form of posterior collapse.

\begin{figure}[t]
\centering
\resizebox{0.98\linewidth}{!}{%
\begin{tikzpicture}[>=Latex, font=\scriptsize]
\tikzstyle{box}=[draw,rounded corners,align=center,minimum height=0.72cm,inner sep=4pt]
\tikzstyle{lossbox}=[draw,rounded corners,align=center,minimum height=0.72cm,inner sep=4pt,thick]
\tikzstyle{optbox}=[draw,dashed,rounded corners,align=center,minimum height=0.72cm,inner sep=4pt]
\node[box, minimum width=1.15cm] (x) at (0,0) {Input\\$x$};
\node[box, minimum width=1.55cm] (enc) at (2.45,0){Encoder\\$q_\phi(z\mid x)$};
\node[box, minimum width=1.85cm] (z) at (4.90,0){Latent sample\\$z\sim q_\phi(z\mid x)$\\Witness route\\$Z_\mu=u_\phi(X)$};
\node[box, minimum width=1.65cm] (dec) at (7.35,0){Decoder\\$p_\theta(x\mid z)$};
\node[box, minimum width=1.20cm] (rec) at (9.85,0){R: Reconstruction\\$\hat x$};
\draw[->] (x) -- (enc); \draw[->] (enc) -- (z); \draw[->] (z) -- (dec); \draw[->] (dec) -- (rec);
\node[box, minimum width=3.20cm] (witness) at (4.90,1.65){S: Fixed Student witness\\$S(\cdot\mid Z_\mu)$};
\node[box, minimum width=2.25cm] (teach) at (0,3.12){T: Frozen Teacher\\$T_X=T(c\mid X)$};
\node[lossbox, minimum width=3.10cm] (align) at (4.90,3.12){Alignment\\$L_{\mathrm{TS}}=\mathbb E\,\KL(T_X\Vert S(\cdot\mid Z_\mu))$};
\node[box, minimum width=2.10cm] (stats) at (0,4.50){Precomputed\\$\bar T,I_T$};
\node[lossbox, minimum width=2.20cm] (cert) at (8.10,4.50){Certificate\\$G_T=I_T-L_{\mathrm{TS}}$};
\draw[->] (x.north) -- (teach.south); \draw[->] (teach.north) -- (stats.south); \draw[->] (z.north) -- (witness.south); \draw[->] (witness.north) -- (align.south); \draw[->] (teach.east) -- (align.west); \draw[->] (stats.east) -- (cert.west); \draw[->] (align.north east) -- (cert.south west);
\node[optbox, minimum width=2.15cm] (rhat) at (10.25,1.65){Optional reconstruction-side\\Teacher posterior $R_{\hat x}$};
\node[optbox, minimum width=2.35cm] (optloss) at (10.25,3.12){Diagnostics\\$L_{TR},L_{SR}$};
\draw[dashed,->] (rec.north) -- (rhat.south); \draw[dashed,->] (rhat.north) -- (optloss.south); \draw[dashed,->] (align.east) -- (optloss.west);
\node[lossbox, minimum width=2.45cm] (vae) at (4.90,-1.55){VAE loss\\reconstruction $+$ KL};
\draw[->] (x.south) |- (vae.west); \draw[->] (rec.south) |- (vae.east);
\end{tikzpicture}}
\caption{RST architecture (R=Reconstruction, S=Student, T=Teacher).  The decoder sample $z$ is distinct from the audited mean $Z_\mu=u_\phi(X)$.  Here $R_{\hat x}$ is the frozen Teacher posterior on the reconstruction; $L_{TR}$ and $L_{SR}$ compare Teacher/Student with it.  These dashed diagnostics are defined in Appendix~\ref{app:decoder-route} and do not enter the certificate.}
\label{fig:rst-main-architecture}
\end{figure}

\section{From the exact boundary to constructive simplex geometry}
\label{sec:main-theory}
\paragraph{Exact route-local certificate.}
In plain terms, this asks whether the fixed reader can beat every input-independent representation.
\begin{theorem}[Information decomposition and attainable input-independent boundary]
\label{thm:main-info-decomp}\label{thm:main-boundary}
For any $C-X-Z$ and witness $S(C\mid Z)$ that is strictly positive on the Teacher support,
\begin{equation}
\begin{aligned}
L_{\rm TS}&=I(C;X\mid Z)+\mathbb E_Z\KL(P(C\mid Z)\|S(C\mid Z)),\\
G_T&=I(C;Z)-\mathbb E_Z\KL(P(C\mid Z)\|S(C\mid Z)).
\end{aligned}
\label{eq:cert-decomp}
\end{equation}
Hence $I(C;Z)\ge G_T$.  If $Z\perp X$, then $P(C\mid Z)=\bar T$ and
\[
L_{\rm TS}=I_T+\mathbb E_Z\KL(\bar T\|S(C\mid Z))\ge I_T.
\]
If the witness admits a constant code $u_0$ with $S(\cdot\mid u_0)=\bar T$---as the centered simplex witness below does at $u_0=0$---the boundary is attained.  Thus, for the RST witness,
\[
\boxed{G_T>0\Longrightarrow I(C;Z)\ge G_T>0,\qquad \inf_{Z\perp X}L_{\rm TS}=I_T.}
\]
\end{theorem}
A nonpositive margin is inconclusive.
\begin{corollary}[Route-local re-instantiation]\label{cor:main-reinstantiation}
Replacing $(C,Z,T,S)$ by $(C_m,Z_m,T^{(m)},S_m)$ gives the same exact boundary and sufficient information margin for any chosen representation/statistic satisfying $C_m-X-Z_m$, provided the witness has a constant code producing the corresponding marginal.
\end{corollary}
The identities also hold exactly on a fixed finite split when $\bar T$, $I_T$, and $L_{\rm TS}$ are computed self-consistently on that empirical distribution; this is not a population-confidence statement.  Proofs and the population extension are in the appendix.

\paragraph{Canonical interface geometry: realization, freedom, and stability.}
The next question is geometric: what is the shortest exact Teacher code, and which directions can vary without changing the fixed readout?  Requiring minimum visible dimension, an explicit witness-null complement, and minimum normalized worst-case sensitivity uniquely selects regular-simplex rows up to rotation among centered unit-row minimal-dimensional witnesses (Appendix~\ref{app:new-constructive-proofs}).  For a realized input $x$, write $T_x:=T(\cdot\mid x)$.  Let $v_1,\dots,v_K$ be unit regular-simplex vertices, $\sum_kv_k=0$, $v_i^\top v_j=-1/(K-1)$ for $i\ne j$, and let $V$ collect their rows.  With gain $\beta>0$ define
\[
S(u)=\softmax(\beta Vu+\log\bar T),\qquad \mathcal V=\operatorname{row}(V),\quad \mathcal N=\ker V=\mathcal V^\perp.
\]
Let $P_{\mathcal V}$ denote the orthogonal projector onto $\mathcal V$.
\begin{proposition}[Canonical simplex interface: minimality, exact fiber, and witness-null invariance]\label{prop:main-minimality}\label{prop:main-minimax-simplex}\label{prop:main-fiber}\label{prop:main-nullspace}
A linear-softmax witness that locally represents all full-support $K$-class categorical distributions needs at least $K-1$ visible dimensions; the simplex attains this minimum.  For full-support $T_x$, define the centered log ratio
\[
r_k(x)=\log\frac{T_{x,k}}{\bar T_k}-\frac1K\sum_j\log\frac{T_{x,j}}{\bar T_j},\qquad
u_T(x)=\frac{K-1}{K\beta}V^\top r(x).
\]
Then
\begin{equation}
\boxed{S(u_T(x))=T_x,\qquad S^{-1}(T_x)=u_T(x)+\ker V,}
\label{eq:fiber}
\end{equation}
$u_T(x)$ is the unique minimum-norm realization, and
\[
\frac12\mathbb E\|u_T(X)\|^2=\frac{K-1}{2K\beta^2}\,\mathcal E_T,
\quad \mathcal E_T=\mathbb E\|r(X)\|^2.
\]
Writing $u_{\mathcal V}:=P_{\mathcal V}u$ and $u_{\mathcal N}:=u-u_{\mathcal V}$, every $u=u_{\mathcal V}+u_{\mathcal N}$ satisfies $S(u)=S(u_{\mathcal V})$: witness-null variation is exactly certificate invariant, although it still pays ordinary VAE KL cost and need not be used by the decoder.  The regular simplex is not required by the certificate theorem; it is a canonical choice because it simultaneously gives minimal visible dimension, class symmetry, an isotropic tight frame, the closed-form inverse/fiber, an orthogonal witness-null complement, and the smallest normalized worst-case sensitivity among unit-norm row-centered minimal-dimensional witnesses.
\end{proposition}

\paragraph{Movement, cost, and function-space departure.}
Next we ask how to leave the centered reference, and at what Gaussian mean-prior energy cost.
\begin{proposition}[Analytic margin--energy path]\label{prop:main-frontier}
Along $u_\alpha(x)=\alpha u_T(x)$, $0\le\alpha\le1$, define $G(\alpha):=I_T-L_{\rm TS}(u_\alpha)$ and $E_{\rm mean}(\alpha):=\tfrac12\mathbb E\|u_\alpha(X)\|^2$.  Then
\begin{equation}
S_k(u_\alpha)=\frac{\bar T_k^{1-\alpha}T_{x,k}^{\alpha}}{\sum_j\bar T_j^{1-\alpha}T_{x,j}^{\alpha}},\qquad
\boxed{G(\alpha)\ge\alpha I_T},\qquad
\boxed{E_{\rm mean}(\alpha)=\alpha^2E_{\rm mean}(1)}.
\label{eq:path}
\end{equation}
$G$ is concave and nondecreasing, with $G(0)=0$, $G(1)=I_T$, and initial slope $\mathbb E[J(T_X,\bar T)]$, where $J(P,Q):=\KL(P\|Q)+\KL(Q\|P)$ is the Jeffreys divergence.
\end{proposition}
\begin{corollary}[Strict function-space escape from the centered input-independent solution]\label{cor:main-transverse}
Assume $I_T>0$ and full support.  Within the class of input-independent codes $u(x)\equiv a$ for a constant vector $a$, the centered origin is a global minimizer of $L_{\rm TS}$ with value $I_T$.  In the full space of input-dependent code functions, however,
\[
\left.\frac{d}{d\alpha}L_{\rm TS}(u_\alpha)\right|_{\alpha=0}
=-\mathbb E_X J(T_X,\bar T)<0,
\]
Therefore, for any Teacher--Student weight $\lambda_{\rm TS}>0$ and any finite quadratic mean-prior coefficient, the combined Teacher--Student plus quadratic mean cost still has a strictly negative right derivative at the origin.
\end{corollary}
\begin{proof}
Any input-independent code induces a constant witness distribution $q$, so $L_{\rm TS}=I_T+\KL(\bar T\|q)\ge I_T$, with equality at the centered witness $S(0)=\bar T$.  Proposition~\ref{prop:main-frontier} gives $G'(0)=\mathbb E J(T_X,\bar T)$ and hence the displayed negative derivative of $L_{\rm TS}=I_T-G$.  The quadratic mean-prior cost is $O(\alpha^2)$ along $u_\alpha=\alpha u_T$, so its right derivative at $0$ is zero.
\end{proof}
The function-space result is not merely post-hoc; the remaining question is whether its code-space escape direction survives the encoder parameterization.  For a parameterized encoder,
\[
\nabla_\phi L_{\rm TS}=\mathbb E[J_\phi u_\phi(x)^\top g_u(x)],
\]
where $J_\phi u_\phi(x)$ is the encoder Jacobian and $g_u(x):=\nabla_u\KL(T_x\|S(u))|_{u=u_\phi(x)}$ is the code-space alignment gradient.  Thus the encoder Jacobian filters the code-space escape field into parameter space: saturation, small witness-visible singular values, opposing ELBO gradients, or poor Jacobian conditioning can attenuate the realized update.  We assume no uniform lower bound on the relevant Jacobian singular values; no optimizer-independent convergence claim is made.

\begin{figure}[htbp]
\centering
\begin{minipage}[t]{0.485\linewidth}
\centering
\begin{tikzpicture}
\begin{axis}[
 width=\linewidth,height=4.1cm,
 xlabel={$\alpha$},ylabel={$G(\alpha)$},
 xmin=0,xmax=1,ymin=0,ymax=4.7,
 grid=major,grid style={gray!18},
 legend to name=analyticPathLegend,
 legend columns=4,
 legend style={font=\scriptsize,draw=none,/tikz/every even column/.append style={column sep=0.8em}},
 tick label style={font=\scriptsize},label style={font=\small}
]
\addplot+[mark=*,thick] coordinates {(0,0) (.05,.3175338) (.10,.6273866) (.20,1.2132628) (.40,2.1512680) (.60,2.6512941) (.80,2.8170229) (1,2.8456666)};
\addlegendentry{coarse path}
\addplot+[dashed,thick] coordinates {(0,0) (1,2.8456666)};
\addlegendentry{coarse certified bound}
\addplot+[mark=square*,thick] coordinates {(0,0) (.05,.4111677) (.10,.8195940) (.20,1.6220519) (.40,3.0774768) (.60,4.0313523) (.80,4.3657539) (1,4.4169271)};
\addlegendentry{fine path}
\addplot+[dotted,thick] coordinates {(0,0) (1,4.4169271)};
\addlegendentry{fine certified bound}
\end{axis}
\end{tikzpicture}
\vspace{-0.7em}
{\scriptsize (a) Certificate margin along the analytic path.}
\end{minipage}\hfill
\begin{minipage}[t]{0.485\linewidth}
\centering
\begin{tikzpicture}
\begin{axis}[
 width=\linewidth,height=4.1cm,
 xlabel={$E_{\mathrm{mean}}(\alpha)$},ylabel={$G(\alpha)$},
 xmin=0,xmax=1.48,ymin=0,ymax=4.7,
 grid=major,grid style={gray!18},
 tick label style={font=\scriptsize},label style={font=\small},
 samples=100
]
\addplot+[mark=*,thick] coordinates {(0,0) (.0021413,.3175338) (.0085652,.6273866) (.0342607,1.2132628) (.1370426,2.1512680) (.3083459,2.6512941) (.5481705,2.8170229) (.8565164,2.8456666)};

\addplot+[dashed,thick,domain=0:0.8565164] {2.8456666*sqrt(x/0.8565164)};

\addplot+[mark=square*,thick] coordinates {(0,0) (.0035383,.4111677) (.0141531,.8195940) (.0566124,1.6220519) (.2264496,3.0774768) (.5095115,4.0313523) (.9057983,4.3657539) (1.4153098,4.4169271)};

\addplot+[dotted,thick,domain=0:1.4153098] {4.4169271*sqrt(x/1.4153098)};

\end{axis}
\end{tikzpicture}
\vspace{-0.7em}
{\scriptsize (b) The same path against mean latent energy.}
\end{minipage}
\vspace{0.05em}
\pgfplotslegendfromname{analyticPathLegend}
\caption{Analytic margin--energy path on CIFAR-100 coarse and fine Teacher views.  Left: the exact computed certificate is above the guaranteed linear lower bound $\alpha I_T$.  Right: eliminating $\alpha$ from $E_{\mathrm{mean}}(\alpha)=\alpha^2E_{\mathrm{mean}}(1)$ yields the certified lower curve $G\ge I_T\sqrt{E_{\rm mean}/E_{\rm mean}(1)}$.  These are evaluations of the closed-form construction, not learned encoder trajectories or claims of a globally Pareto-optimal frontier.}
\label{fig:analytic-margin-energy-audit}
\end{figure}
\vspace{-0.55em}
\paragraph{Preservation and composition.}
Finally, we ask how much visible drift is safe and how several prescribed views can coexist.
\begin{theorem}[Certified visible cylinder]\label{thm:main-cylinder}
For the simplex witness, $L_{\rm simp}:=\beta\sqrt{2K/(K-1)}$ satisfies
\[
|\KL(T_x\|S(u'))-\KL(T_x\|S(u))|\le L_{\rm simp}\|P_{\mathcal V}(u'-u)\|.
\]
Hence a positive-margin anchor $u_{\rm anc}$ with $G_{\rm anc}:=I_T-L_{\rm TS}(u_{\rm anc})>0$ remains certified under any refinement whose visible displacement is uniformly below $G_{\rm anc}/L_{\rm simp}$; witness-null displacement is unrestricted.  At $u_\alpha$, for any $0<\gamma<1$, choosing visible radius $\varepsilon_\alpha=\gamma\alpha I_T/L_{\rm simp}$ guarantees $G\ge(1-\gamma)\alpha I_T>0$.
\end{theorem}

\begin{proposition}[Orthogonal multi-view composition]\label{prop:main-multiview}
For view $m$, write $K_m$ for its number of classes; $V_m$ and $u_{T,m}$ for its witness matrix and analytic Teacher code; and $I_m,\mathcal E_m,\beta_m,G_m$ for its Teacher information, log-ratio energy, witness gain, and certificate margin.  For $M$ Teachers placed in pairwise orthogonal visible blocks $\mathcal V_m$ of dimensions $K_m-1$, the code
\begin{equation}
u(x)=\sum_{m=1}^M\alpha_m u_{T,m}(x)+h(x),\qquad h(x)\in\bigcap_m\ker V_m,
\label{eq:multiview}
\end{equation}
obeys $G_m\ge\alpha_m I_m$ simultaneously, with additive visible energy $\sum_m c_m\alpha_m^2$, $c_m=(K_m-1)\mathcal E_m/(2K_m\beta_m^2)$.  Thus latent dimension $d_z\ge\sum_m(K_m-1)$ is a sufficient capacity rule for independent categorical views, and requested margins $\tau_m$ are achieved constructively by $\alpha_m=\tau_m/I_m$.  Dependent or hierarchical views may admit shared geometries; orthogonality is the clean sufficient construction studied here.
\end{proposition}
\begin{corollary}[Requested-margin energy budget]\label{cor:main-budget}
For requested margins $0\le\tau_m\le I_m$, choosing $\alpha_m=\tau_m/I_m$ guarantees all $G_m\ge\tau_m$ and uses visible mean energy $\sum_m c_m(\tau_m/I_m)^2$.
\end{corollary}

\begin{figure}[H]
\centering
\begin{tikzpicture}[x=1.12cm,y=0.68cm,>=Latex,font=\scriptsize]
  \draw[->] (-0.35,0) -- (5.25,0) node[right] {$\mathcal V$};
  \draw[->] (0,-1.75) -- (0,1.8) node[above] {$\mathcal N=\ker V$};
  \fill (0,0) circle (2pt) node[below left] {$0$};
  \fill (2.55,0) circle (2pt) node[below] {$u_\alpha$};
  \fill (4.55,0) circle (2pt) node[below] {$u_T$};
  \draw[->,thick] (0.15,0) -- (2.38,0); \draw[->,thick] (2.72,0) -- (4.38,0);
  \draw[dashed,thick] (4.55,-1.45)--(4.55,1.45);
  \node[right,align=left] at (4.65,1.0) {exact fiber\\$u_T+\ker V$};
  \fill[gray!14] (2.12,-1.28) rectangle (2.98,1.28);
  \draw[dashed] (2.12,-1.42)--(2.12,1.42); \draw[dashed] (2.98,-1.42)--(2.98,1.42);
  \node[align=center] at (2.55,0.95) {certified\\cylinder};
  \node[below,align=center] at (0,-1.42) {$G=0$\\centered input-indep. point};
  \node[below,align=center] at (2.55,-1.42) {$G\ge\alpha I_T$\\certified point};
\end{tikzpicture}
\caption{One geometry supports realization, routing, movement, and preservation.  The analytic ray lies in the witness-visible space; all exact Teacher realizations share the affine witness-null fiber.  Around a positive-margin anchor, only visible displacement is bounded.}
\label{fig:main-geometry}
\end{figure}
\vspace{-1.10em}
\paragraph{Training objective.}
RST adds the fixed Teacher--Student term to an otherwise standard VAE,
\vspace{-0.25em}
\[
\mathcal L=\mathcal L_{\rm rec}+\beta_{\rm KL}\mathcal L_{\rm KL}+\lambda_{\rm TS}L_{\rm TS}.
\]
\vspace{-0.35em}
Here $\mathcal L_{\rm rec}$ and $\mathcal L_{\rm KL}$ are the standard reconstruction and posterior-KL terms, and $\beta_{\rm KL},\lambda_{\rm TS}\ge0$ are their weights.  The Teacher and witness are frozen in the certified stage; certification is evaluated against the independently defined boundary $I_T$.  A routed implementation may write $u_\phi(x)=\alpha_\star u_T(x)+Nh_\phi(x)$, where $\alpha_\star$ is a fixed or prefit visible scale and the columns of $N$ form a basis of $\ker V$ (hence $VN=0$), controlling the visible route while leaving the common witness-null component learnable.  Teacher construction, full objectives, gradient diagnostics, and implementation details are reported in the appendix.

\vspace{-1.00em}
\section{Experimental evidence}
\label{sec:main-experiments}
\vspace{-0.35em}
We organize the evidence by the claim it tests.  All reported certificates use frozen Teachers and witnesses restricted to the declared representation.  The broad single-view studies use the PCA--GMM Teacher construction described in the appendix; the CIFAR-100 coarse/fine study uses label-defined Teachers specifically to isolate routing and composition rather than Teacher discovery.
\vspace{-0.25em}
\begin{table}[H]
\centering
\footnotesize
\setlength{\tabcolsep}{3.3pt}
\caption{Five-dataset boundary crossing (mean $\pm$ sample s.d.\ over five seeds).  Two fixed RST variants are shown for every dataset, avoiding per-dataset best-variant selection.  $G_T>0$ is sufficient; $G_T\le0$ is inconclusive.}
\label{main:tab-five}\label{tab:five-dataset}
\begin{tabular}{lrrr}
\toprule
Dataset & VAE $G_T$ & RST-$\alpha$-prefit $G_T$ & RST-$\alpha$-logit-prefit $G_T$ \\
\midrule
Fashion-MNIST & $-0.277\pm0.328$ & $\mathbf{1.503\pm0.060}$ & $\mathbf{1.626\pm0.054}$ \\
CIFAR-10 & $-1.987\pm0.360$ & $\mathbf{1.841\pm0.062}$ & $\mathbf{1.836\pm0.041}$ \\
CIFAR-100 & $-2.074\pm0.242$ & $\mathbf{1.488\pm0.056}$ & $\mathbf{1.542\pm0.039}$ \\
SVHN & $-1.088\pm0.133$ & $\mathbf{1.943\pm0.026}$ & $\mathbf{1.967\pm0.029}$ \\
Tiny-ImageNet-200 & $-2.279\pm0.360$ & $\mathbf{1.321\pm0.036}$ & $\mathbf{1.317\pm0.101}$ \\
\bottomrule
\end{tabular}
\end{table}
\vspace{-0.45em}
\paragraph{Boundary crossing, stress tests, escape, and preservation.}
Across five image datasets, vanilla VAE is certificate-negative while both fixed RST variants in Table~\ref{main:tab-five} are positive in every case; stronger KL pressure preserves the sign separation.  Standard warm-up, cyclical annealing, and free-bits controls remain certificate-negative, whereas witness-null routing preserves the predicted logits while leaving reconstruction comparable.  Teacher degradation lowers the margin, shuffled Teachers are weaker, and random Dirichlet Teachers are near zero, confirming that the certificate tracks the \emph{specified} Teacher rather than semantic correctness.  Starting in the collapse neighborhood, escape runs reach $G\approx2.0$ on CIFAR-10, $1.6$--$1.7$ on CIFAR-100, and $1.31$--$1.43$ on Tiny-ImageNet-200; encoder-gradient diagnostics show the nonzero code-space field survives the parameterization at practical scale.  Under later strong-KL refinement, an unconstrained route falls to $G_\tau=-0.671$ with AU$=0$, while certified refinement keeps $G_\tau\ge2.735$ with AU$=64$.  The closed-form CIFAR-100 path is numerically exact to floating-point precision.  Full breadth statistics are in Appendix~\ref{app:main-breadth}.

\paragraph{Activity, decodability, and prescribed alignment.}
Matched CIFAR-100 controls use the same coarse/fine Teachers and $55$-epoch setup.  The fixed-variance $\delta$ model is fully active yet fails both declared-view certificates.  More sharply, Aux-full and Aux-block learn the Teachers without ever using the simplex witness during training; after freezing each encoder, a fresh linear probe achieves positive Teacher-relative margins, while the predeclared simplex margins remain negative in every seed.  RST is positive under both the post-hoc probe and the same predeclared witness.  Thus the information is present in the Aux codes, but not bound to the declared interface.

\begin{table}[H]
\centering
\footnotesize
\setlength{\tabcolsep}{2.7pt}
\caption{CIFAR-100 matched controls (five seeds).  $G_{\rm pre}$ uses the predeclared simplex; $G_{\rm probe}$ fits a linear probe only after freezing the encoder.  A dash means the probe was not part of this diagnostic.}
\label{tab:controls}\label{tab:new-exp3}
\resizebox{\linewidth}{!}{%
\begin{tabular}{lrrr}
\toprule
Method & AU & $G_{\rm pre}^{c/f}$ & $G_{\rm probe}^{c/f}$ \\
\midrule
VAE & $167.8\!\pm\!2.8$ & $-3.344\!\pm\!1.131~/~-4.162\!\pm\!0.396$ & -- \\
InfoVAE & $172.2\!\pm\!10.3$ & $-3.435\!\pm\!0.913~/~-4.650\!\pm\!0.340$ & -- \\
Fixed-var. $\delta$ & $192.0\!\pm\!0.0$ & $-4.718\!\pm\!0.368~/~-5.944\!\pm\!0.193$ & -- \\
Aux-full & $192.0\!\pm\!0.0$ & $-3.089\!\pm\!0.779~/~-4.892\!\pm\!0.377$ & $\mathbf{1.078\!\pm\!0.013~/~1.710\!\pm\!0.028}$ \\
Aux-block & $186.8\!\pm\!2.2$ & $-2.697\!\pm\!1.059~/~-5.136\!\pm\!0.758$ & $\mathbf{0.986\!\pm\!0.013~/~1.619\!\pm\!0.037}$ \\
RST-joint & $191.8\!\pm\!0.4$ & $\mathbf{1.057\!\pm\!0.027~/~1.477\!\pm\!0.026}$ & $\mathbf{1.162\!\pm\!0.016~/~1.957\!\pm\!0.026}$ \\
\bottomrule
\end{tabular}%
}
\end{table}
\vspace{-0.35em}

\paragraph{No-refit reuse across seeds and architectures.}
Within the ConvVAE Aux-block control, native learned-head accuracy is $36.57/21.88\%$ (coarse/fine), but cross-seed no-adapter reuse falls to $5.20/0.94\%$; Procrustes makes $39/40$ off-diagonal margins positive and ridge-affine makes $40/40$.  The fixed RST witness, in contrast, directly reads every Conv-RST seed.  We then change the producer architecture: five ResVAE RST models are trained independently with the same frozen coarse/fine Teachers, the same fixed witness, and the same $d_z=192$ block allocation, but without regression to $u_T(x)$ or the hard routed parameterization.  A predeclared $45$k/$5$k train/validation split selects checkpoints by the fixed final validation objective; the official test is accessed only after selection.

\begin{table}[H]
\centering
\scriptsize
\setlength{\tabcolsep}{2.35pt}
\caption{Cross-architecture fixed-interface audit on CIFAR-100 (five seeds).  ``Fixed acc.'' uses the same predeclared witness; ``native Aux'' uses the jointly trained control head.  Conv checkpoints predate the $45$k/$5$k Res split, so absolute backbone performance is not compared.}
\label{tab:cross-architecture-main}
\resizebox{\linewidth}{!}{%
\begin{tabular}{llrrr}
\toprule
Producer & Method & $G_{\rm pre}^{c/f}$ & fixed acc. $c/f$ (\%) & native Aux acc. $c/f$ (\%) \\
\midrule
ConvVAE & RST & $\mathbf{1.057\pm0.027~/~1.477\pm0.026}$ & $\mathbf{43.84\pm1.30~/~33.58\pm0.56}$ & -- \\
ResVAE & RST & $\mathbf{0.869\pm0.057~/~1.083\pm0.105}$ & $\mathbf{37.28\pm1.94~/~23.87\pm2.20}$ & -- \\
ConvVAE & Aux-block & $-2.697\pm1.059~/~-5.136\pm0.758$ & $4.77\pm0.71~/~1.01\pm0.17$ & $36.57\pm0.85~/~21.88\pm2.24$ \\
ResVAE & Aux-block & $-2.543\pm1.004~/~-5.314\pm0.883$ & $4.90\pm0.75~/~0.85\pm0.18$ & $20.03\pm1.11~/~10.38\pm1.09$ \\
\bottomrule
\end{tabular}%
}
\end{table}
\vspace{-0.45em}

\paragraph{Cross-architecture learned-reader transfer.}
Swapping learned Aux heads across backbones is near chance without refitting: Conv$\to$Res gives $5.11/0.91\%$ and Res$\to$Conv $4.98/0.96\%$ coarse/fine accuracy, with $98/100$ pairwise margins negative.  Adapters are fitted on $20{,}000$ training examples only.  Procrustes makes $89/100$ margins positive, whereas ridge-affine makes all $100/100$ positive and reaches $21.81/14.74\%$ (Conv$\to$Res) and $24.43/10.47\%$ (Res$\to$Conv).  Because Procrustes is incomplete, the mismatch is not explained by a pure rotation.  Full per-direction margins, aggregation, and protocol are in Appendix~\ref{app:cross-architecture-interface}.

\paragraph{Controlled interface audit.}
Two interventions connect the functional result back to the exact geometry.  Rotating an RST visible block while keeping the original witness fixed drops coarse/fine accuracy to $4.81/1.00\%$; co-rotating the readout restores $43.84/33.58\%$ with maximum logit error below $5\times10^{-6}$.  Zeroing either visible block returns only that route to the input-independent boundary ($|G|<5\times10^{-7}$) while the orthogonal route is numerically unchanged; swapping a block makes the target interface follow the donor while the non-target interface follows the recipient.  These are latent-interface audits, not decoder-causality claims; full per-seed results are in Appendix~\ref{app:functional-gauge-intervention}.

\paragraph{Selective multi-view routing and composition.}
CIFAR-100 provides a 20-way coarse and 100-way fine label view.  We smooth each label Teacher with $\epsilon=0.02$ and allocate simplex-visible blocks of dimensions $19$ and $99$ inside $d_z=192$, leaving a $74$-dimensional common witness-null complement.  Five seeds give the selective pattern
\[
\boxed{\text{baseline }(-,-),\quad \text{coarse }(+,-),\quad \text{fine }(-,+),\quad \text{joint }(+,+),}
\]
with joint-null also $(+,+)$ in every seed.  Table~\ref{tab:compose} gives the mean margins and route accuracies.  The experiment intervenes on preassigned routes; it does not test unsupervised semantic discovery.

\begin{table}[H]
\centering
\small
\setlength{\tabcolsep}{3.5pt}
\caption{Five-seed selective protection on CIFAR-100.  Accuracies use the corresponding fixed witnesses.}
\label{tab:compose}
\renewcommand{\arraystretch}{1.08}
\begin{tabular}{lrrrr}
\toprule
Method & $G_c$ & $G_f$ & coarse acc. & fine acc. \\
\midrule
VAE & $-1.660\pm0.964$ & $-1.728\pm0.324$ & $5.45\%$ & $0.85\%$ \\
Coarse only & $\mathbf{0.619\pm0.017}$ & $-2.519\pm0.416$ & $34.54\%$ & $0.88\%$ \\
Fine only & $-2.302\pm1.167$ & $\mathbf{0.579\pm0.013}$ & $5.36\%$ & $20.99\%$ \\
Joint & $\mathbf{0.686\pm0.019}$ & $\mathbf{0.655\pm0.019}$ & $36.94\%$ & $22.54\%$ \\
Joint + witness-null & $\mathbf{0.693\pm0.017}$ & $\mathbf{0.632\pm0.014}$ & $37.70\%$ & $23.59\%$ \\
\bottomrule
\end{tabular}
\end{table}
\vspace{-0.35em}
\section{Related work}
\label{sec:main-related}
\paragraph{Posterior collapse and information objectives.}
Collapse has been linked to decoder dominance or latent underuse~\citep{bowman2016generating,chen2017lossy}, inference lag and optimization~\citep{he2019lagging,dai2020usual}, linear-Gaussian structure~\citep{lucas2019understanding}, oversmoothing~\citep{takida2021oversmoothing}, non-identifiability~\citep{wang2021nonidentifiability}, and explicit mode-wise thresholds~\citep{dang2024beyond,ichikawa2024dynamics,hirn2026spectral,li2026phase}.  Remedies alter KL schedules or free bits~\citep{fu2019cyclical,kingma2016iaf}, inference, posterior constraints, skip paths, decoder regularity, local reconstruction, or reparameterization noise~\citep{he2019lagging,razavi2019delta,zhao2019infovae,tolstikhin2018wae,dieng2019skip,kinoshita2023inverse,song2026local,demisse2026lambda}.  These methods target useful latent information in general.  RST instead fixes a particular Teacher-relative information route.  Variational information bounds~\citep{tishby1999ib,oord2018cpc,poole2019variational,barber2003agakov} supply related information-theoretic tools; the earlier Teacher-guided certificate~\citep{zhang2026teacher} established the best-constant/input-independent boundary used here.

\paragraph{Fixed readouts, probing, and compatible representations.}
Linear probes ask what a separately fitted readout can recover from frozen features~\citep{alain2016probes,hewitt2019probes}.  Neural collapse identifies simplex-ETF geometry in terminal supervised representations~\citep{papyan2020neuralcollapse}.  Separately, fixed-classifier work freezes the terminal readout and lets the feature extractor adapt, including generic fixed classifiers, regular-polytopal/simplex or ETF frames, and hierarchy-aware frames~\citep{hoffer2018fix,pernici2022polytope,yang2022fixedetf,liang2023haframe}.  Compatible-representation work studies direct or transformed reuse across independently trained or updated models~\citep{shen2020backward,gygli2021reusable,biondi2024stationary,ricci2025lambdaorth,biondi2026stationary}, whereas model stitching explicitly learns a connector between frozen networks~\citep{bansal2021stitching}.  Concept-bottleneck models likewise make an intermediate concept interface explicit and intervenable~\citep{koh2020cbm}.  RST does \emph{not} claim fixed simplex classifiers, stationary/compatible representations, or adapter-free reuse as new.  Its distinct object is a Teacher-relative certificate inside a VAE: a frozen possibly soft Teacher specifies the view, a predeclared witness specifies the audited interface, and the same geometry yields an attainable boundary, exact realization/fiber, margin--energy movement, preservation, and multi-view composition.

\paragraph{Simplex geometry.}
Regular simplices and tight frames are classical~\citep{benedetto2003frames}.  The contribution is not the simplex itself, but the RST-specific chain obtained once that classical geometry is coupled to the Teacher-relative certificate: minimum visible dimension, closed-form soft-Teacher inverse, complete witness-null fiber, certified movement/energy, escape, preservation, and composition.  Under the normalized minimal-interface criterion stated in Appendix~\ref{app:new-constructive-proofs}, the classical regular simplex is the canonical realization up to rotation.

\section{Discussion and scope}\label{sec:main-discussion}
\paragraph{What prescribed alignment adds.}
If a downstream task may refit an arbitrary probe, ordinary decodability can suffice; the new readout can absorb coordinate changes.  RST addresses the narrower case in which route identity and readout are part of the specification.  The formal rotation counterexample (Appendix~\ref{app:decodability-separation}) and the cross-seed/cross-architecture transfer studies show why these notions differ: Aux information is present and recoverable, yet direct producer--consumer reuse fails without coordinate adaptation, while one predeclared RST witness remains readable across two producer backbones.  This is a certificate-level interface claim, not a claim of intrinsic semantic coordinates, universal architecture invariance, or a first compatible-representation method.  The functional contrast is external specification versus post-hoc negotiation: RST fixes the consumer geometry before producer training, whereas Aux lets the producer and learned head co-adapt.  Thus the cross-architecture result tests portability of a fixed contract for the tested ConvVAE/ResVAE pair, not representation identity.

\paragraph{Why a certificate rather than another compatibility objective?}
Compatibility objectives can encourage direct reuse, but reuse alone does not quantify which frozen Teacher information crosses which declared route.  RST separates specification from training: the Teacher, witness, and reference information $I_T$ are fixed before the audited optimization, so $G_T>0$ has an information-theoretic meaning that does not depend on fitting a downstream adapter.  The construction then turns the same certificate into an explicit interface contract---what must be visible, what may vary in the witness-null complement, how a positive state may move while remaining certified, and how several declared contracts compose.  The compatibility-style experiment above is evidence for this contract, not the definition of the contribution.

\paragraph{Scope.}
The guarantee is relative to the frozen Teacher; unsupervised Teacher discovery is separate.  $G_T\le0$ is inconclusive, re-instantiation requires the stated Markov condition, and the analytic ray/escape field are function-space statements rather than optimizer-independent convergence guarantees.  Preservation begins from a positive-margin anchor, and neither witness-null invariance nor certification proves decoder use.  RST therefore claims a sufficient, Teacher-relative, route-specific information guarantee---not semantic completeness or universal collapse resolution.

\section{Conclusion}\label{sec:main-conclusion}
RST turns an established Teacher-relative input-independent boundary into a constructive prescribed-information interface.  One certificate/geometry supports exact realization, residual freedom, movement and cost, escape, preservation, composition, and---empirically---no-refit reuse across independently trained producers and a second backbone when the interface is predeclared.  The novelty is therefore neither the classical simplex nor compatibility as a problem, but a certifiable bridge from a declared (possibly soft) information view to a constructive, reusable VAE interface.

\clearpage
\appendix
\raggedbottom
\linespread{0.95}\selectfont
\setlength{\parskip}{1.6pt plus 0.15pt minus 0.15pt}
\setlength{\textfloatsep}{5.5pt plus 0.5pt minus 1.5pt}
\setlength{\floatsep}{4.5pt plus 0.5pt minus 1.5pt}
\setlength{\intextsep}{5pt plus 0.5pt minus 1.5pt}
\setlength{\abovedisplayskip}{4.2pt plus 0.6pt minus 1.8pt}
\setlength{\belowdisplayskip}{4.2pt plus 0.6pt minus 1.8pt}
\setlength{\abovedisplayshortskip}{3pt plus 0.5pt minus 1.5pt}
\setlength{\belowdisplayshortskip}{3pt plus 0.5pt minus 1.5pt}
\setlist{topsep=2pt,itemsep=1pt,parsep=0pt,partopsep=0pt}
\sloppy
\begin{center}{\Large\bfseries APPENDIX}\end{center}
\vspace{-0.35em}
\section{Reader map, notation, and claim--evidence map}
\label{app:reader-map}
This opening block is a lookup and navigation layer.  It defines notation once, maps each main claim to its proof/evidence, and points to the five evidence blocks below.

\subsection*{Notation and terminology}
The main text follows a definition-before-use convention for nonstandard symbols.  The tables below are a lookup aid; symbols used only inside a specialized appendix derivation are defined locally there.

\paragraph{RST name and core probabilistic notation.}
RST always expands to \textbf{Reconstruction--Student--Teacher}: R is the ordinary VAE reconstruction branch, S is the fixed Student witness/readout, and T is the frozen Teacher specification.  R is never used as the mathematical representation variable.  This Teacher/Student terminology differs from knowledge distillation~\citep{hinton2015distill}: here the Teacher is a frozen information specification and the Student is a fixed code-only witness, not a learned predictor compressed from another model.
\begin{center}
\scriptsize
\setlength{\tabcolsep}{4.2pt}
\renewcommand{\arraystretch}{1.10}
\begin{tabular}{p{0.19\linewidth}p{0.25\linewidth}p{0.47\linewidth}}
\toprule
Symbol / term & Name & Meaning in this paper \\
\midrule
$X,x$ & input random variable / realization & Data sample and a realized input. \\
$T_X,T_x$ & frozen Teacher view & $T_X:=T(\cdot\mid X)$ and $T_x:=T(\cdot\mid x)$; a full-support categorical distribution over the declared view. \\
$C$ & Teacher variable & Auxiliary categorical variable sampled as $C\mid X\sim T_X$. \\
$\bar T$ & Teacher marginal & $\mathbb E_X[T_X]$; the optimal constant prediction for the declared Teacher view. \\
$Z,Z_m$ & audited representation & Deterministic or stochastic representation/statistic being audited; distinct from R in RST. \\
$Z_\mu$ & mean representation & $Z_\mu:=u_\phi(X)$; the deterministic encoder mean audited in the main experiments. \\
$z$ & sampled VAE latent & $z\sim q_\phi(z\mid x)$; the sample passed to the decoder $p_\theta(x\mid z)$.  Lowercase $z$ is not used for the deterministic witness code. \\
$S(\cdot\mid Z)$ & Student / fixed witness & Predeclared code-only readout used to audit the chosen representation interface. \\
$I_T$ & Teacher information / boundary & $I_T:=I(C;X)=\mathbb E_X\KL(T_X\|\bar T)$; the attainable input-independent loss boundary for the centered witness. \\
$L_{\rm TS}$ & Teacher--Student loss & $\mathbb E_{X,Z}\KL(T_X\|S(\cdot\mid Z))$. \\
$G_T$ & certificate margin & $G_T:=I_T-L_{\rm TS}$; $G_T>0$ implies $I(C;Z)\ge G_T>0$. \\
input-independent reference & audited reference class & $Z\perp X$; every such representation obeys $L_{\rm TS}\ge I_T$. \\
constant collapse & deterministic special case & $Z\equiv u_0$, hence input-independent. \\
prescribed alignment & fixed-interface requirement & The declared Teacher view must be readable at the predeclared representation/readout interface; $G_T>0$ is a sufficient certificate, not the definition itself. \\
learned decodability & post-hoc readability & The view can be recovered after fitting a new head; the readout coordinates may change. \\
adapter-free interface & functional interpretation & Evaluation or reuse with the predeclared witness, without refitting an adapter/head. \\
\bottomrule
\end{tabular}
\end{center}

\paragraph{Simplex, optimization, and diagnostic notation.}
{
\scriptsize
\setlength{\tabcolsep}{4.2pt}
\renewcommand{\arraystretch}{1.06}
\begin{longtable}{p{0.19\linewidth}p{0.25\linewidth}p{0.47\linewidth}}
\toprule
Symbol & Name & Meaning in this paper \\
\midrule
$K$ & number of Teacher classes & A $K$-class full-support categorical Teacher needs $K-1$ simplex-visible dimensions. \\
$v_k,V,\beta$ & simplex witness geometry & $v_k$ are unit regular-simplex vertices, $V$ stacks them as rows, and $\beta>0$ is the fixed witness gain. \\
$u$ & deterministic witness code & Generic deterministic input to the simplex witness $S(u)=\softmax(\beta Vu+\log\bar T)$. \\
$\mathcal V,P_{\mathcal V}$ & witness-visible space / projector & $\mathcal V=\operatorname{row}(V)$; $P_{\mathcal V}$ is its orthogonal projector. \\
$\mathcal N$ & witness-null space & $\mathcal N=\ker V=\mathcal V^\perp$; geometric directions invisible to the witness, not the same as $Z\perp X$. \\
$r(x)$ & centered log-ratio vector & Centered coordinates of $\log(T_{x,k}/\bar T_k)$ used by the analytic inverse. \\
$u_T(x)$ & analytic Teacher code & Unique minimum-norm visible code satisfying $S(u_T(x))=T_x$; all exact solutions are $u_T(x)+\ker V$. \\
$\mathcal E_T$ & Teacher log-ratio energy & $\mathbb E\|r(X)\|^2$; determines the closed-form mean-prior cost. \\
$u_\alpha,\alpha$ & analytic Teacher path & $u_\alpha(x)=\alpha u_T(x)$, $0\le\alpha\le1$. \\
$G(\alpha)$ & path margin & $I_T-L_{\rm TS}(u_\alpha)$. \\
$E_{\rm mean}(\alpha)$ & mean-prior energy & $\tfrac12\mathbb E\|u_\alpha(X)\|^2$. \\
$N,h_\phi$ & witness-null residual parameterization & Columns of $N$ span $\ker V$; $Nh_\phi(x)$ can vary without changing the fixed witness logits. \\
$\mathcal L_{\rm rec},\mathcal L_{\rm KL}$ & standard VAE terms & Reconstruction loss and posterior-to-prior KL; weights are $\beta_{\rm KL}$ and $\lambda_{\rm TS}$ for KL and Teacher--Student alignment. \\
$d_z$ & latent dimension & Dimension of the VAE latent/mean code used in an experiment. \\
$G_\tau$ & safety-slack margin & $G_\tau:=G_T-\tau$ when a reserved positive slack $\tau$ is imposed in preservation tests. \\
$G_c,G_f$ & coarse / fine margins & CIFAR-100 certificate margins for the 20-way coarse and 100-way fine Teacher views. \\
AU & active units & Generic latent-activity diagnostic; not equivalent to prescribed alignment. \\
$R_{\hat x}$ & reconstruction-side Teacher posterior & Frozen Teacher evaluated on the reconstruction $\hat x$; an optional decoder-side diagnostic object. \\
$L_{TR},L_{SR}$ & reconstruction-side diagnostics & Teacher--Reconstruction and Student--Reconstruction KL diagnostics; neither enters $G_T$. \\
\bottomrule
\end{longtable}
}

\subsection*{Appendix roadmap}
The appendix is evidence-first and split into six blocks.  Appendix~\ref{app:reader-map} provides notation and a claim--evidence map.  Appendix~\ref{app:core-proofs} places the core proofs and constructive geometry in dependency order.  Appendix~\ref{app:interface-separation} joins the exact decodability--alignment separation with the matched interface experiments.  Appendix~\ref{app:full-evidence} contains full empirical tables and reproducibility details.  Appendix~\ref{app:mechanistic-audits} collects mechanistic and numerical audits.  Appendix~\ref{app:extensions-scope} states route re-instantiations and limitations, with prospective extensions explicitly separated from validated claims.

\subsection*{Claim--evidence map}
The table separates exact mathematical claims from the experiment that probes each consequence.  The analytic path audit checks implementation of closed-form identities; it is not treated as evidence that SGD follows the analytic ray.
\begin{center}
\scriptsize
\setlength{\tabcolsep}{4pt}
\renewcommand{\arraystretch}{1.08}
\begin{tabular}{p{0.30\linewidth}p{0.25\linewidth}p{0.35\linewidth}}
\toprule
Main claim & Formal support & Empirical check / role \\
\midrule
$G_T$ lower-bounds prescribed information & Thm.~\ref{thm:main-info-decomp} & Positive margins on five datasets \\
$I_T$ is the exact attainable input-independent boundary for the centered RST witness & Thm.~\ref{thm:main-boundary} & Constant/input-independent controls attain the boundary; RST crosses it \\
The certificate re-instantiates for a chosen route/statistic & Cor.~\ref{cor:main-reinstantiation} & Block-wise studies; log-variance cross-statistic sanity check \\
$K-1$ dimensions are minimal; simplex is normalized minimax-stable & Props.~\ref{prop:main-minimality}, \ref{prop:main-minimax-simplex} & Geometric statement; no learned-witness premise \\
Exact Teacher solutions form $u_T+\ker(V)$ and $u_T$ is minimum norm & Prop.~\ref{prop:main-fiber} & Closed-form inverse numerical audit; routed runs \\
Analytic path has concave margin and quadratic energy & Prop.~\ref{prop:main-frontier} & Analytic--computation audit only (Fig.~\ref{fig:analytic-margin-energy-audit}) \\
Witness-null motion is invariant; positive anchors have certified visible cylinders & Prop.~\ref{prop:main-nullspace}, Thm.~\ref{thm:main-cylinder} & Witness-null-route runs; trajectory-mined certified refinement \\
Multiple orthogonal views compose with additive energy budgets & Prop.~\ref{prop:main-multiview}, Cor.~\ref{cor:main-budget} & CIFAR-100 coarse/fine selective sign pattern; moving-shape block test \\
Generic activity does not imply the prescribed certificates & Certificate definition plus matched controls & $\delta$ control: $192/192$ AU but both margins negative \\
Learned decodability does not imply prescribed alignment & Prop.~\ref{prop:decodability-separation} plus fixed-witness definition & Exact rotation counterexample; frozen probes are positive while fixed margins are negative; cross-seed/cross-architecture reuse fails and adapters recover \\
\bottomrule
\end{tabular}
\end{center}
\noindent\textit{Fast verification route.}  Core mathematical support is in Appendix~\ref{app:core-proofs}; interface separation in Appendix~\ref{app:interface-separation}; full empirical evidence in Appendix~\ref{app:full-evidence}; implementation-facing audits in Appendix~\ref{app:mechanistic-audits}; and extensions/scope in Appendix~\ref{app:extensions-scope}.

\FloatBarrier
\section{Core proofs and constructive geometry}
\label{app:core-proofs}
\mainback{sec:main-theory}{Section~\ref*{sec:main-theory} (constructive theory)}
This block follows the main theoretical dependency: certificate and boundary $\rightarrow$ canonical simplex realization $\rightarrow$ exact fiber $\rightarrow$ path/energy $\rightarrow$ preservation $\rightarrow$ multi-view composition.  Classical geometric facts are used as ingredients; the claims here are the RST-specific consequences and their proofs.

\subsection{Certificate, boundary, and core construction proofs}
This section records the proofs that are only stated compactly in the main text.  It is deliberately separate from the later auxiliary identities and numerical audits.

\paragraph{Theorem~\ref{thm:main-info-decomp} (information decomposition and input-independent boundary).}
Because $C-X-Z$, $P(C\mid X,Z)=P(C\mid X)=T(C\mid X)$.  Insert $P(C\mid Z)$ into the alignment loss:
\[
\begin{aligned}
L_{\rm TS}
&=\mathbb E\log\frac{P(C\mid X,Z)}{S(C\mid Z)}\\
&=\mathbb E\log\frac{P(C\mid X,Z)}{P(C\mid Z)}
+\mathbb E\log\frac{P(C\mid Z)}{S(C\mid Z)}\\
&=I(C;X\mid Z)+\mathbb E_Z\KL(P(C\mid Z)\|S(C\mid Z)).
\end{aligned}
\]
Also $I_T=I(C;X)$.  Since $I(C;Z\mid X)=0$, the chain rule gives
$I(C;X)=I(C;Z)+I(C;X\mid Z)$.  Subtracting the displayed identity from $I_T$ yields
\[
G_T=I(C;Z)-\mathbb E_Z\KL(P(C\mid Z)\|S(C\mid Z))\le I(C;Z).
\]
If $Z\perp X$, then $Z\perp C$ and $P(C\mid Z)=\bar T$, so
\[
L_{\rm TS}=I_T+\mathbb E_Z\KL(\bar T\|S(C\mid Z))\ge I_T.
\]
For the centered RST witness, the constant code $u=0$ satisfies $S(\cdot\mid0)=\bar T$, hence the boundary is attained.

\paragraph{Corollary~\ref{cor:main-reinstantiation} (route-local re-instantiation).}
Replace $(C,Z,T,S)$ by $(C_m,Z_m,T^{(m)},S_m)$ in the preceding argument.  No step uses a special property of the posterior mean; only the stated Markov relation and witness support condition are required.

\paragraph{Proposition~\ref{prop:main-fiber} (minimal realization, exact fiber, and witness-null invariance).}
Softmax is invariant to a common logit shift, so identifiable categorical logits lie in $\mathbf1^\perp$, a $(K-1)$-dimensional space.  A linear-softmax witness that represents an open set of full-support $K$-class distributions therefore needs visible rank at least $K-1$.  The regular simplex has exactly this rank and satisfies
\[
VV^\top=\frac{K}{K-1}\left(I-\frac1K\mathbf1\mathbf1^\top\right),\qquad
V^\top V=\frac{K}{K-1}P_{\mathcal V}.
\]
For the centered log-ratio vector $r(x)\in\mathbf1^\perp$, the definition
$u_T(x)=\frac{K-1}{K\beta}V^\top r(x)$ gives $\beta Vu_T(x)=r(x)$ and hence $S(u_T(x))=T_x$.
Conversely, if $S(u)=T_x$, centering the equal logits removes the common shift and gives $V(u-u_T)=0$, so the full solution set is $u_T+\ker V$.  Because $u_T\in\mathcal V$ and $\ker V=\mathcal V^\perp$, $u_T$ is the unique minimum-norm realization.  Finally, $Vh=0$ for every $h\in\ker V$, so $S(u+h)=S(u)$ and the certificate is exactly invariant to witness-null motion.

\paragraph{Proposition~\ref{prop:main-frontier} (analytic margin--energy path).}
For one sample define $D_x=\KL(T_x\|\bar T)$, $a_k(x)=\log(T_{x,k}/\bar T_k)$, and
$A_x(\alpha)=\log\sum_k\bar T_ke^{\alpha a_k(x)}$.  Along $u_\alpha=\alpha u_T$,
\[
\ell_x(\alpha):=\KL(T_x\|S(u_\alpha))
=(1-\alpha)D_x+A_x(\alpha),
\qquad
\ell_x''(\alpha)=\operatorname{Var}_{S(u_\alpha)}a(x)\ge0.
\]
Thus $L_{\rm TS}(\alpha)$ is convex and nonincreasing, so
$G(\alpha)=I_T-L_{\rm TS}(\alpha)$ is concave and nondecreasing.  Since
$G(0)=0$ and $G(1)=I_T$, concavity gives $G(\alpha)\ge\alpha I_T$.
Differentiating at the origin yields
\[
G'(0)=\mathbb E_X\!\left[\KL(T_X\|\bar T)+\KL(\bar T\|T_X)\right].
\]
The energy identity follows immediately from $u_\alpha=\alpha u_T$:
$E_{\rm mean}(\alpha)=\alpha^2E_{\rm mean}(1)$.

\paragraph{Theorem~\ref{thm:main-cylinder} (certified visible cylinder).}
For $\ell_x(u)=\KL(T_x\|S(u))$,
\[
\nabla_u\ell_x(u)=\beta V^\top(S(u)-T_x)\in\mathcal V.
\]
Since $\|S(u)-T_x\|_2\le\sqrt2$ and $\|V\|_2=\sqrt{K/(K-1)}$,
$\|\nabla_u\ell_x(u)\|_2\le L_{\rm simp}:=\beta\sqrt{2K/(K-1)}$.
The mean-value inequality therefore gives
\[
|\ell_x(u')-\ell_x(u)|
\le L_{\rm simp}\|P_{\mathcal V}(u'-u)\|_2.
\]
Averaging over $x$ yields the stated margin bound.  If the displacement lies in
$\mathcal N=\ker V$, the right-hand side is zero exactly.

\paragraph{Proposition~\ref{prop:main-multiview} and Corollary~\ref{cor:main-budget} (orthogonal composition).}
Pairwise orthogonality makes witness $m$ insensitive to every other visible block and to the common witness-null complement.  Applying the single-view path bound blockwise gives
$G_m\ge\alpha_m I_m$ simultaneously.  Pythagoras gives additive visible mean energy
$\sum_m c_m\alpha_m^2$.  Setting $\alpha_m=\tau_m/I_m$ therefore achieves every requested margin
$G_m\ge\tau_m$ with the stated additive budget.

\subsection{Constructing the simplex vertices}
\label{app:simplex}

The simplex geometry is fixed once $K$ and $d_z$ are chosen. It is not learned. In $\R^K$, let $e_1,\ldots,e_K$ be the standard basis and define
\[
u_k=e_k-\frac1K\mathbf 1.
\]
Then $\sum_k u_k=0$ and the points lie in the $(K-1)$-dimensional hyperplane orthogonal to $\mathbf 1$. After normalization,
\[
\tilde v_k=\frac{u_k}{\|u_k\|},
\]
one obtains
\[
\tilde v_i^\top\tilde v_j=
\begin{cases}
1,&i=j,\\
-\frac1{K-1},&i\ne j.
\end{cases}
\]
If $d_z>K-1$, the simplex is embedded into any fixed $(K-1)$-dimensional subspace of $\R^{d_z}$, with the remaining coordinates set to zero. Thus the difficult design choice is the teacher posterior $T_x$, not the simplex vertices.

\subsection{Path, inverse, and certified-cylinder proofs}
\label{app:new-constructive-proofs}

\subsubsection{Tight-frame identities and the minimum-norm inverse}
For completeness, the regular-simplex Gram matrix is
\[
VV^\top=\frac{K}{K-1}\left(I-\frac1K\mathbf1\mathbf1^\top\right).
\]
Its nonzero eigenvalues are all $K/(K-1)$, so the rows form a tight frame for $\mathcal V$ and $V^\top V=\frac{K}{K-1}P_{\mathcal V}$.  These identities simultaneously give the inverse in Proposition~\ref{prop:main-fiber}, its energy, and the Lipschitz constant in Theorem~\ref{thm:main-cylinder}.  They also show why an arbitrary learned linear witness does not generally retain the same closed-form geometry: unequal singular values destroy the scaled-isometry property.

The normalized optimality can be stated as an equivalence that makes the design choice non-arbitrary.
\begin{proposition}[Canonicality under the normalized minimal-interface criterion]
\label{prop:canonical-simplex-equivalence}
Let $W\in\R^{K\times(K-1)}$ satisfy $W^\top\mathbf1=0$, $\operatorname{rank}(W)=K-1$, and unit row norms.  The following are equivalent: (i) $\|W\|_2^2=K/(K-1)$, the minimum possible spectral norm under these constraints; (ii) all nonzero singular values equal $\sqrt{K/(K-1)}$; (iii) $WW^\top=\frac{K}{K-1}(I-\frac1K\mathbf1\mathbf1^\top)$; and (iv) the rows are regular-simplex vertices up to a right orthogonal transformation.  Hence the regular simplex is the unique normalized minimal-dimensional witness geometry, up to rotation, that minimizes the worst-case linear sensitivity entering the certified-cylinder radius.
\end{proposition}
\begin{proof}
The unit-row constraint gives $\|W\|_F^2=K$, while rank $K-1$ gives $K-1$ nonzero singular values.  Therefore their largest square is at least their average $K/(K-1)$, with equality iff all are equal.  Centering and full rank make $\mathbf1$ the unique left-null direction, so equality forces $WW^\top=\frac{K}{K-1}P_{\mathbf1^\perp}$, whose diagonal entries are one and off-diagonal entries are $-1/(K-1)$.  This is the regular-simplex Gram matrix.  Conversely that Gram matrix has the stated singular values and spectral norm.  Right multiplication by an orthogonal matrix preserves the Gram matrix and accounts for the remaining coordinate freedom.
\end{proof}
This classical tight-frame characterization is not claimed as a new frame-theoretic result.  Its role here is to show that the same normalized geometry selected by the RST interface criteria is also the geometry that supports the closed-form inverse/fiber and the explicit preservation constant.

\subsubsection{Jeffreys slope and concavity of the certificate path}
For one sample let $D_x=\KL(T_x\|\bar T)$ and $a_k=\log(T_{x,k}/\bar T_k)$.  With
\[
A_x(\alpha)=\log\sum_k\bar T_k e^{\alpha a_k},
\]
we have
\[
\ell_x(\alpha)=(1-\alpha)D_x+A_x(\alpha),\qquad
\ell_x''(\alpha)=\operatorname{Var}_{q_{\alpha,x}}a\ge0.
\]
At the origin,
\[
A_x'(0)=\sum_k\bar T_k\log\frac{T_{x,k}}{\bar T_k}
=-\KL(\bar T\|T_x),
\]
so
\[
-\ell_x'(0)=D_x+\KL(\bar T\|T_x)=J(T_x,\bar T).
\]
At $\alpha=1$, $q_{1,x}=T_x$ and $\ell_x'(1)=0$.  Hence the improvement is steepest at the collapsed endpoint and gradually saturates toward exact Teacher matching.

\subsubsection{A closed-form certificate--energy achievability bound}
For a requested margin $0<\tau<I_T$, Proposition~\ref{prop:main-frontier} gives the explicit choice $\alpha=\tau/I_T$.  The resulting code satisfies $G(\alpha)\ge\tau$ and has exactly
\[
E_{\mathrm{mean}}(\alpha)=
\left(\frac{\tau}{I_T}\right)^2E_{\mathrm{mean}}(1).
\]
This is an achievability statement for the analytic ray, not a claim that the displayed energy is globally minimal over arbitrary nonlinear code functions.

\subsubsection{Average-radius variant of the certified cylinder}
The main theorem uses a uniform radius for a hard per-sample guarantee.  The same proof gives
\[
G(u_0+\delta)\ge G_0-L_{\mathrm{simp}}\E_X\|P_{\mathcal V}\delta(X)\|_2.
\]
Thus an average visible-displacement budget also preserves the empirical/population margin in expectation.  The uniform version is preferable when the goal is to define a strict feasible set for every sample.

\subsection{Scaled analytic targets and the margin-energy tradeoff}
\label{app:alpha-target}

The closed-form code $\uT(x)$ gives exact matching, $S(\uT(x))=T_x$. It is often preferable to use a scaled target
\[
u(x)=\alpha \uT(x),\qquad 0\le \alpha\le1,
\]
which trades certificate strength against proximity to the Gaussian prior center. Along this ray,
\[
S_k(\alpha \uT(x))=
\frac{T_{x,k}^{\alpha}\bar T_k^{1-\alpha}}
{\sum_j T_{x,j}^{\alpha}\bar T_j^{1-\alpha}}.
\]
Thus $\alpha=0$ gives the constant predictor $\bar T$, while $\alpha=1$ gives the teacher $T_x$. Let
\[
Z_\alpha(x)=\sum_j T_{x,j}^{\alpha}\bar T_j^{1-\alpha}.
\]
Then
\[
\KL(T_x\|S(\alpha \uT(x)))
=(1-\alpha)\KL(T_x\|\bar T)+\log Z_\alpha(x).
\]
Hence $L_{\mathrm{TS}}(0)=I_T$ and $L_{\mathrm{TS}}(1)=0$.

The path is continuous. It is also nonincreasing in $\alpha$ on $[0,1]$. To see this, write
\[
\ell_k(x)=\log\frac{T_{x,k}}{\bar T_k},
\qquad
A_x(\alpha)=\log\sum_j \bar T_j\exp(\alpha \ell_j(x)).
\]
Then
\[
S_\alpha(k\mid x)=\bar T_k\exp\{\alpha\ell_k(x)-A_x(\alpha)\},
\]
and the log-partition function satisfies
\[
A_x'(\alpha)=\E_{S_\alpha(x)}\ell(x),\qquad
A_x''(\alpha)=\operatorname{Var}_{S_\alpha(x)}(\ell(x))\ge0.
\]
Thus $A_x'$ is nondecreasing. Since $S_1(x)=T_x$, we have $A_x'(1)=\E_{T_x}\ell(x)$, and therefore $A_x'(\alpha)\le A_x'(1)$ for $0\le\alpha\le1$. Differentiating the KL expression gives
\[
\frac{d}{d\alpha}\KL(T_x\|S(\alpha \uT(x)))
=
-\E_{T_x}\ell(x)+A_x'(\alpha)
\le0.
\]
In nondegenerate cases the inequality is strict except on flat intervals where the relevant log-odds are constant.

For any margin $0<\tau<I_T$, at least one $\alpha$ satisfies
\[
L_{\mathrm{TS}}(\alpha)\le I_T-\tau.
\]
Hence the safety scale $\alpha_\star(\tau)$ defined above is well defined; it is the minimum-energy point on this analytic ray that still achieves the desired certificate margin. Its mean-KL energy is
\[
\frac12\E_x\|\alpha_\star(\tau) \uT(x)\|^2
=\alpha_\star(\tau)^2\frac{K-1}{2K\beta^2}\mathcal E_T.
\]
The design interpretation is direct: move only as far from the prior center as needed to beat the constant baseline.

The same interpretation can be used to quantify partial non-collapse inside the protected witness subspace. The scalar $\alpha$ does not create new active latent directions; it only scales the teacher directions already present in $\uT(x)$. Thus, for $\alpha>0$, the rank of the mean-code covariance along this ray is the same as the rank of $\operatorname{Cov}(\uT(x))$. What improves continuously is not full dimensional usage, but the amount of teacher information expressed by the fixed witness. This distinction is important: $L_{\mathrm{TS}}\approx0$ means $S(u(x))\approx T_x$, not that every coordinate or every posterior dimension is noncollapsed.

\subsection{Distance to the analytic target and its effect on alignment}
\label{app:target-distance}

Let
\[
d(x)=\uphi(x)-\uT(x).
\]
Since $S(\uT(x))=T_x$, the witness-logit perturbation caused by $d(x)$ is
\[
\delta a_k(x)=\beta v_k^\top d(x).
\]
Therefore
\[
\KL(T_x\|S(\uT(x)+d(x)))
=
\log\sum_j T_{x,j}e^{\delta a_j(x)}-
\sum_k T_{x,k}\delta a_k(x).
\]
For small $d(x)$,
\[
\KL(T_x\|S(\uT(x)+d(x)))
\approx
\frac12\delta a(x)^\top
\left[\operatorname{Diag}(T_x)-T_x T_x^\top\right]
\delta a(x).
\]
A simple upper bound follows from the bounded curvature of the softmax log-partition:
\[
L_{\mathrm{TS}}
\le
\frac{\beta^2K}{2(K-1)}
\E_x\|\uphi(x)-\uT(x)\|^2.
\]
Thus controlling the witness-relevant distance to $\uT(x)$ controls the alignment loss. The converse is not generally true in the full latent space, since directions orthogonal to the simplex span may not change $S$.

\subsection{General linear-witness safe tubes and trajectory-mined preservation}
\label{app:general-safe-tube}
The fixed simplex gives the explicit constant $L_{\mathrm{simp}}$.  A corresponding statement holds for any fixed linear-softmax witness $S(u)=\softmax(Wu+b)$.  Row-center the matrix,
\[
W_c=\left(I-\frac1K\mathbf1\mathbf1^\top\right)W.
\]
Common-logit shifts are irrelevant, and
\[
\nabla_u\KL(T_x\|S(u))=W_c^\top(S(u)-T_x),
\qquad
\|\nabla_u\KL(T_x\|S(u))\|_2\le\sqrt2\|W_c\|_2=:K_W.
\]
If an anchor $\mu_0$ has safety margin $G_0=I_T-L_{\mathrm{TS}}(\mu_0)-\tau>0$, every refinement $\mu=\mu_0+\delta$ with $\sup_x\|\delta(x)\|\le \rho$ satisfies
\[
G_\tau(\mu)\ge G_0-K_W\rho.
\]
For a Gaussian encoder whose variance is held fixed during this refinement, $\|\delta(x)\|_2$ is also the pointwise $W_2$ distance between the source and refined diagonal posteriors.  Thus $\rho<G_0/K_W$ defines a posterior tube that excludes complete input-independent posterior collapse without requiring an escape gradient and without putting the witness loss in the refinement objective.

\subsubsection{Validation-only anchor selection and final stress test}
One ordinary CIFAR-100 ConvVAE trajectory used latent dimension 64 and was trained from random initialization for 50 epochs with $\beta_{\mathrm{KL}}=10^{-4}$, followed by strong-KL stress with $\beta_{\mathrm{KL}}=5$.  Candidate checkpoints were predeclared.  On each candidate, the encoder was frozen, the witness was fitted on the training split, and validation-only selection maximized the normalized certified radius subject to a positive margin and a PSNR tolerance.  Test data were not used to select the anchor.

The selected epoch was 50.  A quick scan gave $G_{0,\mathrm{val}}=1.90583$, normalized radius $0.0227936$, and PSNR $20.5282$ dB.  After the checkpoint was fixed, the final three-seed witness protocol gave
\[
G_{0,\mathrm{val}}=2.77111,\qquad \rho=0.0403527,
\]
with a held-out finite-split analytical lower bound $1.35285>0$.  The chosen radius is one half of the final certified radius $G_0/K_W$.

\begin{table}[htbp]
\centering
\small
\caption{Detailed trajectory-mined preservation experiment.}
\label{tab:app-certified-tube-final}
\begin{tabular}{lrrrrrr}
\toprule
Branch & Final $G_\tau$ & Min. $G_\tau$ & KL & AU & Max disp. & Stress obj.\\
\midrule
Frozen anchor & 2.73842 & 2.73842 & 54.2750 & 64 & 0 & 271.386\\
Unconstrained & -0.670815 & -0.670815 & $1.67735\times10^{-6}$ & 0 & 16.0919 & 0.0700051\\
Certified tube & 2.73512 & 2.73490 & 54.0151 & 64 & 0.0403527 & 270.086\\
\bottomrule
\end{tabular}
\end{table}

The tube had zero constraint violations.  Relative to the frozen source it reduced the exact stress objective by $1.29947$, with reconstruction-MSE increase $2.95431\times10^{-5}$, KL reduction $0.2599$, and PSNR change $-0.011927$ dB.  The residual was strongly sample dependent:
\[
\|\E\delta(X)\|_2=0.00310921,\qquad
\text{centered residual RMS}=0.0402327,
\]
with centered energy fraction $0.994063$, 64 active residual dimensions, and effective rank $59.0465$.  A common translation would have centered energy fraction zero.

\paragraph{Interpretation.}
This experiment establishes preservation rather than creation.  The early VAE trajectory creates the first informative checkpoint; validation-only mining discovers it; the certificate converts it into a nonzero region that allows further optimization while excluding complete Teacher-relative collapse.  The unconstrained branch demonstrates why the constraint is nontrivial: under strong KL pressure it crosses the certificate boundary and becomes completely collapsed under the declared diagnostics.

\subsection{Two-view transfer and information--energy--stability Teacher design}\label{app:two-view-transfer}
\label{app:view-gap-restored}
A useful extension considers the case in which the Teacher used for training is computed from a different controlled view than the one used for evaluation.  Let $Q_x$ be the training-view Teacher and $P_x$ the evaluation-view Teacher, with
\[
\begin{aligned}
I_Q&=\E\KL(Q_x\|\bar Q), & I_P&=\E\KL(P_x\|\bar P),\\
L_Q&=\E\KL(Q_x\|S(u(x))), & L_P&=\E\KL(P_x\|S(u(x))).
\end{aligned}
\]
Then, defining
\[
\Delta_{XT}:=|I_P-I_Q|+|L_P-L_Q|,
\]
the evaluation-view margin obeys the deterministic inequality
\[
I_P-L_P\ge I_Q-L_Q-\Delta_{XT}.
\]
For the analytic code that exactly matches $Q_x$, $L_Q=0$ and $L_P=\E\KL(P_x\|Q_x)$, so the pre-training quantity
\[
\Delta_{XT}^{\mathrm{analytic}}
=|I_P-I_Q|+\E\KL(P_x\|Q_x)
\]
is directly computable.  A practical Teacher-search score can therefore balance
\[
\text{information } I_Q,
\qquad
\text{minimum exact energy }\frac{K-1}{2K\beta^2}\mathcal E_Q,
\qquad
\text{view stability }\Delta_{XT}^{\mathrm{analytic}}.
\]
This score is a design heuristic rather than a theorem about downstream performance.  The main experiments use the cleaner same-view case, where the gap is zero.

\subsection{A population concentration bound for a frozen marginal}
\label{app:population-concentration}
The main benchmark claims are exact on the empirical evaluation distribution.  If a population statement is desired, a simple bound is available under an explicit positivity condition.  Assume the population teacher marginal $P_C=\bar T$ is fixed and, for some $\eta\in(0,1)$,
\[
\bar T_c\ge\eta,\qquad S(c\mid Z)\ge\eta\quad\text{almost surely for every }c.
\]
For i.i.d.\ evaluation pairs $(X_i,Z_i)$ define
\[
Y_i=\sum_c T(c\mid X_i)\log\frac{S(c\mid Z_i)}{\bar T_c},\qquad
\widetilde G_n=\frac1n\sum_{i=1}^n Y_i.
\]
Then $\E Y_i=G_T$ by the Barber--Agakov form and $|Y_i|\le \log(1/\eta)$.  Hoeffding's inequality therefore gives, with probability at least $1-\delta$,
\[
G_T\ge \widetilde G_n-\log(1/\eta)\sqrt{\frac{2\log(2/\delta)}{n}}.
\tag{A.1}
\]
Thus a strictly positive right-hand side certifies positive population teacher information.  This bound is deliberately separate from the main empirical-split certificate: it assumes a frozen population marginal rather than replacing it by the same-split empirical $\bar T_n$.  More refined sample-splitting or empirical-Bernstein analyses are possible but are not needed for the paper's benchmark claims.

\subsection{Certificate identities and consistency checks}
\paragraph{Direct Barber--Agakov form.}
Using $\IT=\E\log(T(C\mid X)/P_C(C))$ and the definition of $\LTS$ gives directly
\[
\G=\E_{C,Z}\log\frac{S(C\mid Z)}{P_C(C)},
\]
which is the variational MI lower bound.  Theorem~\ref{thm:main-info-decomp} refines it by identifying the exact calibration gap.

\paragraph{Empirical-distribution exactness.}
For a fixed evaluation set $\{x_i\}_{i=1}^n$, take $P_X$ uniform over those points and define $\bar T=n^{-1}\sum_iT_{x_i}$.  All main identities then hold exactly on this finite probability space.  This distinguishes exact certification on the benchmark split from generalization to an unseen population.

\label{app:new-info-proof}

\subsubsection{The constant boundary follows directly from the information identity}
If $Z\equiv u_0$, then $I(C;Z)=0$.  Theorem~\ref{thm:main-info-decomp} therefore gives
\[
G_T=-\KL(\bar T\|S(\cdot\mid u_0))\le0,
\]
which is identical to the direct KL decomposition in Theorem~\ref{thm:main-boundary}.  The centered witness at $u_0=0$ makes the right-hand side zero.  Thus the word ``exact'' can be read either from the best-constant-predictor decomposition or from the information identity.

\subsubsection{Why the Teacher information equals the input--Teacher mutual information}
By construction $P(C=c\mid X=x)=T_x(c)$ and $P_C(c)=\bar T_c$, so
\[
I(C;X)=\E_X\sum_c P(c\mid X)\log\frac{P(c\mid X)}{P_C(c)}
=\E_X\KL(T_X\|\bar T)=I_T.
\]
Thus teacher information is ordinary mutual information of the induced random variable $C$, not an analogy or proxy definition.

\subsubsection{Consistency with data processing}
Because $C\leftarrow X\rightarrow Z$ and the experimental route is generated from $X$, data processing gives $I(C;Z)\le I(C;X)=I_T$.  Together with the main theorem,
\[
G_T\le I(C;Z)\le I_T.
\]
Hence a positive population certificate naturally lies in $0<G_T\le I_T$.  A substantial numerical value $G_T>I_T$ should therefore trigger an implementation, normalization, or estimation audit.

\subsubsection{First-order alignment versus second-order mean-prior cost}
The simplex energy identity gives
\[
\frac12\E\|u_T(X)\|^2=\frac{K-1}{2K\beta^2}\mathcal E_T.
\]
For $u_\alpha=\alpha u_T$, the standard-Gaussian mean-prior cost is exactly
\[
\frac12\E\|u_\alpha(X)\|^2
=\alpha^2\frac{K-1}{2K\beta^2}\mathcal E_T.
\]
Meanwhile Proposition~\ref{prop:main-frontier} gives $L'_{\mathrm{TS}}(0)<0$ whenever $I_T>0$.  The witness-alignment improvement is therefore first order at the origin while the mean-prior penalty has no first-order term.  This is the clean local content of the escape-force statement; the reconstruction term of the full ELBO can still contribute at first order.

\FloatBarrier
\section{Decodability versus prescribed alignment}
\label{app:interface-separation}
\noindent{\footnotesize\emph{\hyperref[sec:main-discussion]{Back to main text: Discussion and scope}; \hyperref[sec:main-experiments]{Experimental evidence}.}}\par\vspace{-0.35em}
This block isolates the distinction between information that a newly fitted head can decode and information that is available through the predeclared interface.  It places the exact rotation separation next to the matched Aux/RST controls, cross-seed and cross-architecture reuse tests, adapter recovery, and controlled gauge checks.

\subsection{Why learned decodability does not imply prescribed alignment}
\label{app:decodability-separation}
The distinction used in the main text has a simple exact form.  A learned probe is allowed to absorb a change of latent coordinates; a prescribed interface is not.  The following construction shows that the implication can fail even when Teacher recovery is exact.

\begin{proposition}[Decodability--alignment separation]
\label{prop:decodability-separation}
Let the centered simplex witness be
\[
S_V(u)=\softmax(\beta Vu+\log\bar T),\qquad \mathcal V=\operatorname{row}(V),\quad \mathcal N=\ker V,
\]
and let $u_T(x)\in\mathcal V$ be the exact Teacher code, so $S_V(u_T(x))=T_x$.  For every orthogonal matrix $Q$, the transformed code $\tilde u(x)=Q u_T(x)$ is still exactly decodable by the correspondingly transformed head
\[
S_Q(\tilde u)=\softmax(\beta VQ^\top\tilde u+\log\bar T),
\qquad S_Q(\tilde u(x))=T_x.
\]
If $\dim\mathcal N\ge\dim\mathcal V=K-1$, one can choose $Q$ with $Q\mathcal V\subseteq\mathcal N$.  For that choice,
\[
S_V(\tilde u(x))=\bar T,\qquad L_{\rm TS}(S_V,\tilde u)=I_T,\qquad G_T(S_V,\tilde u)=0,
\]
while the rotated learned head still has exact Teacher alignment, $L_{\rm TS}(S_Q,\tilde u)=0$.
\end{proposition}
\begin{proof}
Orthogonality gives $Q^\top\tilde u=u_T$, so the transformed head reproduces the original logits.  When $Q\mathcal V\subseteq\ker V$, $V\tilde u=0$ and the original fixed witness reduces to $\softmax(\log\bar T)=\bar T$.  The constant-boundary identity then gives $L_{\rm TS}=I_T$ and $G_T=0$.
\end{proof}
This is a logical separation, not a claim that unconstrained learned representations typically rotate exactly into $\ker V$.  Its role is to identify what a post-hoc probe can hide: coordinate changes can preserve perfect decodability while destroying a previously declared interface.  The Aux-full/Aux-block controls in Table~\ref{tab:controls} test the corresponding practical question without constructing $Q$ by hand.

\subsection{Auxiliary-supervision and functional-interface controls}
\label{app:aux-control-full}
\mainback{sec:main-experiments}{Section~\ref*{sec:main-experiments} (functional-interface evidence)}
Table~\ref{tab:aux-control-perseed} records all ten auxiliary-control runs.  The learned-head columns measure the task each control actually optimizes.  The fixed-witness columns are adapter-free post-training audits using the same predeclared simplex geometry as RST; no new readout is fitted for these columns.  Thus a gap between the two is informative rather than contradictory: it indicates that the labels are readable in the latent code but not in the independently prescribed orientation.

\begin{table}[htbp]
\centering
\scriptsize
\caption{Per-seed auxiliary-supervision controls.  $a^{\mathrm{aux}}$ is accuracy of the learned auxiliary head; $a^{\mathrm{pre}}$ and $G$ use the predeclared fixed simplex witness after training, with no refitted adapter.}
\label{tab:aux-control-perseed}
\begin{tabular}{llrrrrrrr}
\toprule
Method & Seed & $a_c^{\mathrm{aux}}$ & $a_f^{\mathrm{aux}}$ & $G_c$ & $G_f$ & $a_c^{\mathrm{pre}}$ & $a_f^{\mathrm{pre}}$ & AU \\
\midrule
Aux-full & 0 & 0.4033 & 0.2737 & -3.8858 & -4.7483 & 0.0599 & 0.0093 & 192 \\
Aux-full & 1 & 0.4035 & 0.2765 & -2.5292 & -4.5347 & 0.0365 & 0.0100 & 192 \\
Aux-full & 2 & 0.3943 & 0.2566 & -3.0721 & -4.6013 & 0.0449 & 0.0123 & 192 \\
Aux-full & 3 & 0.4163 & 0.2915 & -3.8300 & -5.1843 & 0.0631 & 0.0065 & 192 \\
Aux-full & 4 & 0.4084 & 0.2692 & -2.1256 & -5.3920 & 0.0602 & 0.0129 & 192 \\
\midrule
Aux-block & 0 & 0.3727 & 0.2202 & -2.7727 & -4.4938 & 0.0564 & 0.0082 & 184 \\
Aux-block & 1 & 0.3680 & 0.2308 & -1.3214 & -4.7288 & 0.0380 & 0.0119 & 186 \\
Aux-block & 2 & 0.3605 & 0.2380 & -3.8074 & -4.5395 & 0.0523 & 0.0105 & 187 \\
Aux-block & 3 & 0.3737 & 0.2244 & -3.6090 & -5.8882 & 0.0444 & 0.0086 & 190 \\
Aux-block & 4 & 0.3536 & 0.1805 & -1.9743 & -6.0291 & 0.0475 & 0.0114 & 187 \\
\bottomrule
\end{tabular}
\end{table}

\begin{table}[htbp]
\centering
\small
\caption{Five-seed summary for the auxiliary controls and the existing joint-RST reference.}
\label{tab:aux-control-summary-app}
\begin{tabular}{lrrrrr}
\toprule
Method & $G_c$ & $G_f$ & KL & AU & PSNR \\
\midrule
Aux-full & $-3.089\pm0.779$ & $-4.892\pm0.377$ & $55.74\pm1.12$ & $192.0\pm0.0$ & $20.004\pm0.052$ \\
Aux-block & $-2.697\pm1.059$ & $-5.136\pm0.758$ & $58.32\pm2.36$ & $186.8\pm2.2$ & $20.021\pm0.130$ \\
RST-joint & $1.057\pm0.027$ & $1.477\pm0.026$ & $56.89\pm0.82$ & $191.8\pm0.4$ & $20.108\pm0.102$ \\
\bottomrule
\end{tabular}
\end{table}

For Aux-full, only $8.99\%\pm0.68\%$ of the coarse-head weight energy lies in the RST coarse block, while $52.47\%\pm1.60\%$ of the fine-head weight energy lies in the RST fine block.  Aux-block has fraction one by construction.  These diagnostics are secondary: the central control is the independent fixed-witness audit, which remains negative for both views in all ten auxiliary runs.  We do not use auxiliary-control runtime in the efficiency claim of Experiment~3.

\subsubsection{Functional-interface evaluation: probes, cross-seed reuse, and adapters}
\label{app:functional-interface-full}
All encoders/decoders are frozen for this evaluation.  A true post-hoc linear softmax probe is fitted on the CIFAR-100 training split and evaluated on the test split using the same soft Teacher KL and $I_T$ scale as the certificate.  Aux-block and RST probes read the declared $19/99$-dimensional blocks; Aux-full reads all $192$ mean coordinates.  The probe is diagnostic only and never alters the saved representation.

\begin{table}[htbp]
\centering
\scriptsize
\caption{Frozen-encoder route-feasibility summary (five seeds).  Positive $G_{\rm probe}$ with negative $G_{\rm pre}$ means Teacher information is readable by a newly fitted route but not certified on the predeclared interface.}
\label{tab:functional-probe-app}
\begin{tabular}{llrrrr}
\toprule
Method & View & $G_{\rm probe}$ & probe acc. (\%) & $G_{\rm pre}$ & predeclared acc. (\%) \\
\midrule
Aux-block & coarse & $0.986\pm0.013$ & $38.91\pm0.52$ & $-2.697\pm1.059$ & $4.77\pm0.71$ \\
Aux-block & fine & $1.619\pm0.037$ & $29.31\pm0.70$ & $-5.136\pm0.758$ & $1.01\pm0.17$ \\
Aux-full & coarse & $1.078\pm0.013$ & $43.10\pm0.33$ & $-3.089\pm0.779$ & $5.29\pm1.16$ \\
Aux-full & fine & $1.710\pm0.028$ & $32.09\pm0.54$ & $-4.892\pm0.377$ & $1.02\pm0.26$ \\
RST-joint & coarse & $1.162\pm0.016$ & $45.35\pm0.37$ & $1.057\pm0.027$ & $43.84\pm1.30$ \\
RST-joint & fine & $1.957\pm0.026$ & $37.48\pm0.54$ & $1.477\pm0.026$ & $33.58\pm0.56$ \\
\bottomrule
\end{tabular}
\end{table}

For cross-seed reuse, source seed $s$ contributes its jointly trained Aux-block head and target seed $t$ contributes its frozen encoder.  Diagonal pairs $(s=t)$ are native; off-diagonal pairs use no adapter.  To avoid treating the $20$ ordered off-diagonal pairs as independent replicates, we first average the four source heads for each target, then report mean $\pm$ sample s.d. over the five target seeds.  Every one of the $40$ off-diagonal coarse/fine Teacher margins is negative.

For recovery, paired deterministic train features fit an origin-preserving orthogonal Procrustes map
\[
R^\star=\arg\min_{R^\top R=I}\|X_tR-X_s\|_F^2
\]
using $20{,}000$ examples~\citep{schonemann1966procrustes}; a ridge-affine map is a more flexible supplementary diagnostic.  Procrustes makes $39/40$ transfer margins positive, while ridge-affine makes all $40/40$ positive.  Because affine recovery is consistently stronger, the evidence supports a general coordinate/interface mismatch and should not be read as showing that independently trained Aux codes differ only by a rotation.

\begin{table}[htbp]
\centering
\small
\caption{Cross-seed reuse and adapter recovery.  Accuracy statistics use the per-target aggregation described above.  $G$ columns use the same Teacher-relative margin.}
\label{tab:functional-adapter-app}
\begin{tabular}{llrr}
\toprule
View & Interface & Accuracy (\%) & $G$ \\
\midrule
coarse & Aux native & $36.57\pm0.85$ & $0.866\pm0.021$ \\
coarse & no adapter & $5.20\pm0.54$ & $-0.967\pm0.438$ \\
coarse & +Procrustes & $28.88\pm3.87$ & $0.540\pm0.213$ \\
coarse & +ridge-affine & $35.70\pm0.43$ & $0.826\pm0.006$ \\
fine & Aux native & $21.88\pm2.24$ & $1.178\pm0.116$ \\
fine & no adapter & $0.94\pm0.10$ & $-1.075\pm0.245$ \\
fine & +Procrustes & $16.59\pm1.04$ & $0.939\pm0.052$ \\
fine & +ridge-affine & $21.17\pm0.29$ & $1.132\pm0.018$ \\
\bottomrule
\end{tabular}
\end{table}

\subsubsection{Cross-architecture fixed-interface and reader-transfer audit}
\label{app:cross-architecture-interface}
The cross-architecture audit keeps the CIFAR-100 Teacher definitions, smoothing $\epsilon=0.02$, witness gain $\beta=5$, latent dimension $192$, and the $19/99$ coarse/fine visible blocks fixed.  Five ResVAE seeds are trained for both RST-joint and Aux-block.  For the new ResVAE runs, each fine class contributes $450$ training and $50$ validation examples (a deterministic $45{,}000/5{,}000$ split with split seed $20260905$); the official $10{,}000$-image test set is not accessed until the minimum predeclared validation objective has locked the checkpoint.  Training uses AdamW, 55 epochs, five-epoch KL warmup, and cosine learning rate $2\times10^{-4}\to10^{-5}$.  RST uses only reconstruction, KL, and the two fixed-witness alignment losses with $\lambda_c=\lambda_f=2$: no analytic-code regression and no hard parameterization $u=\alpha u_T+Nh$ are used.  The earlier ConvVAE checkpoints were trained on the original $50$k training set and are reused unchanged; therefore this audit tests whether a second architecture can satisfy the same fixed interface, not a matched-data Conv-versus-Res performance ranking.

\begin{table}[htbp]
\centering
\scriptsize
\caption{Cross-architecture fixed-interface summary (mean $\pm$ sample s.d. over five seeds).  Native Aux columns use each control's jointly trained head; fixed columns use the common predeclared simplex witness.}
\label{tab:cross-architecture-app}
\resizebox{\linewidth}{!}{%
\begin{tabular}{llrrrrrr}
\toprule
Producer & Method & $G_c$ & $G_f$ & fixed $a_c$ (\%) & fixed $a_f$ (\%) & native $a_c$ (\%) & native $a_f$ (\%) \\
\midrule
ConvVAE & RST & $1.057\pm0.027$ & $1.477\pm0.026$ & $43.84\pm1.30$ & $33.58\pm0.56$ & -- & -- \\
ResVAE & RST & $0.869\pm0.057$ & $1.083\pm0.105$ & $37.28\pm1.94$ & $23.87\pm2.20$ & -- & -- \\
ConvVAE & Aux-block & $-2.697\pm1.059$ & $-5.136\pm0.758$ & $4.77\pm0.71$ & $1.01\pm0.17$ & $36.57\pm0.85$ & $21.88\pm2.24$ \\
ResVAE & Aux-block & $-2.543\pm1.004$ & $-5.314\pm0.883$ & $4.90\pm0.75$ & $0.85\pm0.18$ & $20.03\pm1.11$ & $10.38\pm1.09$ \\
\bottomrule
\end{tabular}%
}
\end{table}

For learned-reader transfer, every ordered source-head/target-encoder combination is evaluated ($5\times5$ pairs in each architecture direction and view).  To avoid treating the five source heads for one target as independent replicates, Table~\ref{tab:cross-architecture-adapter-app} first averages source heads within each target and then reports mean $\pm$ sample s.d. over the five target seeds.  Both adapters use $20{,}000$ paired feature examples drawn only from the common training subset; neither validation nor test examples are used for adapter fitting.

\begin{table}[htbp]
\centering
\scriptsize
\setlength{\tabcolsep}{2.1pt}
\caption{Cross-architecture learned-reader reuse and recovery.  Each accuracy/$G$ pair is target-aggregated over five seeds as described in the text.}
\label{tab:cross-architecture-adapter-app}
\resizebox{\linewidth}{!}{%
\begin{tabular}{llrrrrrr}
\toprule
Transfer & View & no-adapter acc. & no-adapter $G$ & Proc. acc. & Proc. $G$ & affine acc. & affine $G$ \\
\midrule
Conv$\to$Res & coarse & $5.11\pm0.50$ & $-0.993\pm0.404$ & $17.22\pm2.41$ & $0.017\pm0.282$ & $21.81\pm1.27$ & $0.458\pm0.033$ \\
Conv$\to$Res & fine & $0.91\pm0.14$ & $-0.816\pm0.138$ & $9.92\pm0.57$ & $0.489\pm0.045$ & $14.74\pm0.64$ & $0.777\pm0.036$ \\
Res$\to$Conv & coarse & $4.98\pm0.50$ & $-0.120\pm0.069$ & $21.69\pm2.90$ & $0.323\pm0.021$ & $24.43\pm0.40$ & $0.317\pm0.003$ \\
Res$\to$Conv & fine & $0.96\pm0.04$ & $-0.125\pm0.028$ & $10.33\pm0.86$ & $0.528\pm0.052$ & $10.47\pm0.15$ & $0.501\pm0.004$ \\
\bottomrule
\end{tabular}%
}
\end{table}

At the raw pair level, $98/100$ no-adapter cross-architecture margins are negative; the two positive values are only $0.0046$ and $0.0232$.  Orthogonal Procrustes makes $89/100$ positive, while ridge-affine makes all $100/100$ positive.  The asymmetric and incomplete Procrustes recovery is evidence for a broader coordinate/interface mismatch, not for a pure rotation.  Conversely, positive fixed-witness margins for all ten Conv/Res RST runs show that two independently optimized producer architectures can satisfy the same predeclared interface; this is evidence for cross-architecture realizability of the specified interface, not a universal architecture-invariance theorem.

\subsubsection{Controlled gauge and selective block interventions}
\label{app:functional-gauge-intervention}
For a controlled coordinate-gauge sanity check, rotate an RST visible block as $X'=XQ$.  Keeping the old fixed witness makes accuracy collapse to $4.81\%$ (coarse) and $1.00\%$ (fine), whereas transforming the readout as $V'=VQ$ restores the original $43.84\%$ and $33.58\%$ accuracies.  The adapted logits match the originals to at most $4.8\times10^{-6}$ across the five seeds.  This is a controlled algebraic check of the separation proposition, not a claim that natural SGD produces an exact orthogonal rotation.

Zero/swap interventions audit the block contract.  Zeroing the coarse block drives its margin to $5.8\times10^{-8}$ while leaving the fine margin at $1.477\pm0.026$; zeroing the fine block drives its margin to $-4.6\times10^{-7}$ while leaving the coarse margin at $1.057\pm0.027$.  Swapping a target block with a donor preserves donor-interface accuracy ($43.84\pm1.30\%$ coarse; $33.58\pm0.56\%$ fine), destroys agreement with the recipient on that route ($4.84\pm0.19\%$; $0.88\pm0.04\%$), and leaves the orthogonal non-target margin numerically unchanged.  These are latent-interface interventions, not decoder-causality evidence.

\FloatBarrier
\section{Full empirical evidence and reproducibility}
\label{app:full-evidence}
\mainback{sec:main-experiments}{Section~\ref*{sec:main-experiments} (Experimental evidence)}
This block contains the complete tables, protocols, and implementation details underlying the compact main-text results.  It is intended as the first stop for numerical verification of reported means, seed stability, matched controls, and reproducibility.

\subsection{CIFAR-100 routing, anti-collapse, and efficiency studies}
\label{app:new-cifar100}
\paragraph{Selective routing and composition.}
The CIFAR-100 coarse labels define $K_c=20$ and the fine labels define $K_f=100$.  Their regular-simplex blocks occupy $19$ and $99$ mutually orthogonal coordinates of a $d_z=192$ code, leaving $74$ common witness-null coordinates.  The Teacher smoothing parameter is $0.02$, witness gain is $5$, batch size is $256$, and the stress stage uses $\beta_{\mathrm{KL}}=5$.  The five variants are baseline, coarse-only, fine-only, joint, and joint with explicit null routing.  Every reported main-table value is mean$\pm$sample standard deviation over seeds $0$--$4$.

\paragraph{Matched anti-collapse comparison.}
The matched comparison uses the same convolutional backbone, latent dimension $192$, batch size $256$, 55 epochs, evaluation every five epochs, AdamW learning rate $2\times10^{-4}$, weight decay $10^{-4}$, KL weight $1$ with five-epoch warmup, and AMP.  InfoVAE uses a multi-scale IMQ-MMD estimator with $\alpha=0$, $\lambda=1000$, and scales $\{0.1,0.2,0.5,1,2,5,10\}$.  Before freezing this baseline, an MMD gradient sanity test verified that optimization reduced a large translated-distribution MMD to near zero and that $\mathrm{MMD}(x,x)=0$ numerically.  An initial $\alpha=1$ diagnostic configuration was rejected before formal multi-seed reporting because, for this simple decoder and estimator scaling, conditional KL and aggregate-posterior MMD became unstable.  Hyperparameters for the reported InfoVAE were fixed from its own optimization stability rather than from RST certificate margins.  The $\delta$ control fixes posterior variance so that the zero-mean conditional KL equals 8 nats; nonzero means can only increase that KL.  It is an architecture-matched committed-rate control, not an exact reproduction of the original autoregressive $\delta$-VAE architecture.  RST-joint uses coarse and fine alignment weights of 2.

\paragraph{Auxiliary-supervision control.}
We add two learned-head controls to separate the supervision signal from the prescribed simplex route.  Both use the same ConvVAE, latent dimension $192$, smoothed coarse/fine Teachers, $55$ epochs, batch size $256$, five seeds, KL schedule, and weights $\lambda_c=\lambda_f=2$ as the joint RST comparison.  Instead of a fixed simplex witness, the auxiliary loss is computed through trainable linear softmax heads.  Aux-full reads all $192$ latent-mean coordinates.  Aux-block restricts the coarse head to the $19$-dimensional coarse block and the fine head to the $99$-dimensional fine block, matching the RST block allocation.  The frozen simplex witnesses are excluded from control training and are used only after training for independent certificate and fixed-route-accuracy audits.  We interpret high learned-head accuracy with nonpositive fixed-route margins as decodability without certified realization on the prescribed route; the audit margin is not used for hyperparameter selection.

\paragraph{Efficiency environment.}
The 20 anti-collapse runs report system metadata from the run records.  The profiling environment was AutoDL vGPU-32GB $\times1$; CUDA reported NVIDIA GeForce RTX 4080 with $31.473$ GiB visible memory.  Driver version was 580.105.08; the host utility reported CUDA 13.0 while the PyTorch binary was built against CUDA 12.8.  Python was 3.12.3, PyTorch 2.8.0+cu128, torchvision 0.23.0+cu128, and cuDNN 9.10.2.  The host had Intel Xeon Platinum 8260 CPUs and 503 GiB system RAM.  Peak memory in Table~\ref{tab:new-exp3} is $\texttt{torch.cuda.max\_memory\_allocated}$, not total driver-level GPU memory usage.

\paragraph{Interpretive discipline.}
A positive $G_c$ or $G_f$ is sufficient evidence for the corresponding Teacher-relative information route.  A non-positive value is not evidence of literal mutual-information zero.  Active units are reported as a generic latent-use diagnostic and are not equated with the certificate.  High common-null mean energy is a geometric routing statistic only; without an explicit decoder-use intervention it is not assigned a semantic interpretation.

\subsubsection{Complete five-seed statistics for the CIFAR-100 studies}
\label{app:new-cifar100-complete-stats}
Tables~\ref{tab:app-exp1-full}--\ref{tab:app-exp2-perseed} report the full five-seed summaries used to construct the main-text tables.  Wide tables are split by inferential role so certificate/interface quantities remain readable without extreme scaling.  These values are not additional hyperparameter selections; they are measurements from the same frozen five-seed runs.

\begin{table}[htbp]
\centering
\footnotesize
\setlength{\tabcolsep}{3pt}
\caption{Five-seed selective routing/composition: certificate and interface diagnostics.  ``Null frac.'' is the posterior-mean energy fraction in the common witness-null complement; it is geometric and is not a decoder-use claim.}
\label{tab:app-exp1-full}
\begin{tabular}{lrrrrr}
\toprule
Method & $G_c$ & $G_f$ & Acc. $c$ & Acc. $f$ & Null frac.\\
\midrule
VAE & $-1.660\pm0.964$ & $-1.728\pm0.324$ & $0.0545\pm0.0053$ & $0.0085\pm0.0018$ & $0.420\pm0.078$\\
RST-coarse & $+0.619\pm0.017$ & $-2.519\pm0.416$ & $0.3454\pm0.0091$ & $0.0088\pm0.0018$ & $0.458\pm0.081$\\
RST-fine & $-2.302\pm1.167$ & $+0.579\pm0.013$ & $0.0536\pm0.0026$ & $0.2099\pm0.0062$ & $0.729\pm0.135$\\
RST-joint & $+0.686\pm0.019$ & $+0.655\pm0.019$ & $0.3694\pm0.0083$ & $0.2254\pm0.0093$ & $0.948\pm0.004$\\
RST-joint-null & $+0.693\pm0.017$ & $+0.632\pm0.014$ & $0.3770\pm0.0080$ & $0.2359\pm0.0114$ & $0.954\pm0.005$\\
\bottomrule
\end{tabular}
\end{table}

\begin{table}[htbp]
\centering
\footnotesize
\setlength{\tabcolsep}{3pt}
\caption{Five-seed selective routing/composition: reconstruction and generic latent-use diagnostics.  Missing MSE entries were not archived for the routed variants.}
\label{tab:app-exp1-recon}
\begin{tabular}{lrrrrr}
\toprule
Method & Rec. & MSE & KL & PSNR & AU\\
\midrule
VAE & $1749.526\pm1.968$ & $0.018024\pm0.000250$ & $20.614\pm0.318$ & $17.442\pm0.060$ & $36.2\pm2.3$\\
RST-coarse & $1748.231\pm2.853$ & -- & $21.249\pm0.846$ & $17.488\pm0.101$ & $49.2\pm4.1$\\
RST-fine & $1749.124\pm1.439$ & -- & $20.819\pm0.375$ & $17.457\pm0.043$ & $39.4\pm1.5$\\
RST-joint & $1746.727\pm2.322$ & -- & $22.029\pm0.907$ & $17.536\pm0.077$ & $57.8\pm4.8$\\
RST-joint-null & $1748.554\pm2.366$ & -- & $21.483\pm0.514$ & $17.484\pm0.069$ & $52.4\pm6.6$\\
\bottomrule
\end{tabular}
\end{table}

\begin{table}[htbp]
\centering
\footnotesize
\setlength{\tabcolsep}{3pt}
\caption{Matched anti-collapse comparison: prescribed-view certificates, fixed-interface accuracy, and activity.}
\label{tab:app-exp2-full-a}
\begin{tabular}{lrrrrr}
\toprule
Method & $G_c$ & $G_f$ & Acc. $c$ & Acc. $f$ & AU\\
\midrule
VAE & $-3.344\pm1.131$ & $-4.162\pm0.396$ & $0.0519\pm0.0022$ & $0.0084\pm0.0010$ & $167.8\pm2.8$\\
InfoVAE & $-3.435\pm0.913$ & $-4.650\pm0.340$ & $0.0479\pm0.0025$ & $0.0105\pm0.0013$ & $172.2\pm10.3$\\
$\delta$-control & $-4.718\pm0.368$ & $-5.944\pm0.193$ & $0.0506\pm0.0024$ & $0.0094\pm0.0013$ & $192.0\pm0.0$\\
RST-joint & $+1.057\pm0.027$ & $+1.477\pm0.026$ & $0.4385\pm0.0130$ & $0.3358\pm0.0056$ & $191.8\pm0.4$\\
\bottomrule
\end{tabular}
\end{table}

\begin{table}[htbp]
\centering
\small
\setlength{\tabcolsep}{6pt}
\caption{Matched anti-collapse comparison: prior diagnostics.  These columns are diagnostic and are not themselves certificates.}
\label{tab:app-exp2-diag}
\begin{tabular}{lrrr}
\toprule
Method & KL & min-KL & MMD\\
\midrule
VAE & $55.155\pm0.866$ & $27.552\pm0.628$ & $0.0590\pm0.0040$\\
InfoVAE & $54.027\pm0.375$ & $28.061\pm0.965$ & $0.0733\pm0.0075$\\
$\delta$-control & $67.734\pm2.326$ & $31.190\pm1.287$ & $0.1276\pm0.0025$\\
RST-joint & $56.885\pm0.817$ & $28.250\pm1.190$ & $0.0621\pm0.0033$\\
\bottomrule
\end{tabular}
\end{table}

\begin{table}[htbp]
\centering
\small
\setlength{\tabcolsep}{6pt}
\caption{Matched anti-collapse comparison: reconstruction and null-routing diagnostics.}
\label{tab:app-exp2-recon-diag}
\begin{tabular}{lrrr}
\toprule
Method & PSNR & MSE & Null frac.\\
\midrule
VAE & $20.022\pm0.042$ & $0.009951\pm0.000096$ & $0.434\pm0.039$\\
InfoVAE & $19.891\pm0.072$ & $0.010256\pm0.000171$ & $0.385\pm0.056$\\
$\delta$-control & $19.597\pm0.105$ & $0.010975\pm0.000267$ & $0.385\pm0.007$\\
RST-joint & $20.108\pm0.102$ & $0.009756\pm0.000231$ & $0.926\pm0.001$\\
\bottomrule
\end{tabular}
\end{table}

\begin{table}[htbp]
\centering\scriptsize
\caption{Complete five-seed reconstruction and efficiency summary for the matched anti-collapse comparison.  Peak values are PyTorch allocator measurements.}
\label{tab:app-exp2-full-b}
\resizebox{\linewidth}{!}{%
\begin{tabular}{lrrrrrr}
\toprule
Method & Rec. & Runtime (s) & Throughput (img/s) & Peak alloc. (GiB) & Peak reserved (GiB) & Params\\
\midrule
VAE & $1686.795\pm0.811$ & $98.43\pm2.32$ & $29226\pm707$ & $0.2086\pm0.0002$ & $0.2813\pm0.0000$ & 3,682,435\\
InfoVAE & $1689.348\pm1.360$ & $137.03\pm3.08$ & $20715\pm464$ & $0.2095\pm0.0002$ & $0.2656\pm0.0000$ & 3,682,435\\
$\delta$-control & $1696.112\pm2.148$ & $97.91\pm2.26$ & $29410\pm699$ & $0.1981\pm0.0010$ & $0.2645\pm0.0070$ & 3,682,435\\
RST-joint & $1685.473\pm1.965$ & $101.25\pm3.73$ & $28371\pm1086$ & $0.2090\pm0.0002$ & $0.2656\pm0.0000$ & 3,682,435\\
\bottomrule
\end{tabular}%
}
\end{table}

\begin{table}[htbp]
\centering\scriptsize
\caption{Per-seed certificate audit for the matched anti-collapse comparison.  The $(-,-)$ versus $(+,+)$ distinction is stable in every seed.}
\label{tab:app-exp2-perseed}
\begin{tabular}{llrrrrr}
\toprule
Method & seed & $G_c$ & $G_f$ & AU & KL & PSNR\\
\midrule
VAE & 0 & -4.5005 & -4.1361 & 168 & 55.065 & 20.0184\\
VAE & 1 & -2.6751 & -4.0955 & 165 & 54.750 & 20.0238\\
VAE & 2 & -3.7656 & -3.6560 & 165 & 54.175 & 20.0427\\
VAE & 3 & -4.0600 & -4.7680 & 171 & 55.266 & 20.0681\\
VAE & 4 & -1.7178 & -4.1520 & 170 & 56.518 & 19.9555\\
InfoVAE & 0 & -3.8727 & -4.3208 & 174 & 53.749 & 19.8620\\
InfoVAE & 1 & -1.9744 & -4.4195 & 162 & 54.168 & 19.8630\\
InfoVAE & 2 & -3.1825 & -4.6296 & 173 & 53.770 & 19.9777\\
InfoVAE & 3 & -4.3325 & -5.1959 & 188 & 54.626 & 19.9508\\
InfoVAE & 4 & -3.8128 & -4.6851 & 164 & 53.824 & 19.8000\\
$\delta$ & 0 & -4.6535 & -6.0753 & 192 & 67.295 & 19.7004\\
$\delta$ & 1 & -4.6218 & -6.0239 & 192 & 68.301 & 19.6351\\
$\delta$ & 2 & -4.9686 & -5.6838 & 192 & 66.323 & 19.5386\\
$\delta$ & 3 & -4.1934 & -5.8007 & 192 & 65.349 & 19.4435\\
$\delta$ & 4 & -5.1543 & -6.1369 & 192 & 71.401 & 19.6680\\
RST-joint & 0 & +1.0341 & +1.4668 & 192 & 57.526 & 20.0921\\
RST-joint & 1 & +1.0238 & +1.4417 & 192 & 55.547 & 19.9384\\
RST-joint & 2 & +1.0652 & +1.4813 & 192 & 57.260 & 20.1821\\
RST-joint & 3 & +1.0858 & +1.5122 & 192 & 56.679 & 20.1809\\
RST-joint & 4 & +1.0758 & +1.4825 & 191 & 57.415 & 20.1486\\
\bottomrule
\end{tabular}
\end{table}

\subsubsection{Breadth evidence moved from the main text}
\label{app:main-breadth}
\begin{table}[htbp]
\centering
\small
\setlength{\tabcolsep}{3.0pt}
\caption{Evidence beyond five-dataset boundary crossing.  AU denotes active units.}
\label{tab:rst-breadth-app}
\begin{tabular}{>{\raggedright\arraybackslash}p{0.18\linewidth}>{\raggedright\arraybackslash}p{0.28\linewidth}>{\raggedright\arraybackslash}p{0.45\linewidth}}
\toprule
Test & Setting & Result \\
\midrule
Strong-KL robustness & $\beta_{\rm KL}=4$ on CIFAR-10/100 & VAE $-0.637/-0.547$ versus RST-$\alpha$-logit-prefit $+1.317/+0.853$ \\
Witness-null routing & routed CIFAR-10/100, 5 seeds & $G=0.276/0.102$, logit residual $0.001$, reconstruction close to VAE; decoder use is not inferred \\
Standard anti-collapse & warm-up, cyclical annealing, free bits & all remain certificate-negative on CIFAR-10/100 ($G\approx-1.63$ to $-2.26$) \\
Teacher correspondence & degraded, shuffled, random & margin decreases with degradation; random Dirichlet Teachers are near zero ($-0.004/-0.005$) \\
Route re-instantiation & moving-shape blocks; log-variance statistic & VAE $G=(-3.707,-2.422)$ versus Block-RST $(0.909,0.468)$; variance check $G=0,0.240,2.728,2.394$ \\
Optimization escape & CIFAR-10/100/Tiny-ImageNet & final $G\approx2.0$, $1.6$--$1.7$, $1.31$--$1.43$ \\
Preservation & CIFAR-100 strong-KL refinement & unconstrained $G_\tau=-0.671$, AU$=0$; certified refinement keeps $G_\tau\ge2.735$, AU$=64$ \\
\bottomrule
\end{tabular}
\end{table}

\subsection{Experimental protocol and reproducibility details}
\label{app:protocol}
\label{app:reproducibility-details}
This section restores the implementation details used by the reported experiments.  The theorem-relevant distinction is architectural rather than cosmetic: reconstruction uses a sampled VAE latent, whereas certification reads only the designated code/statistic through a fixed code-only witness.

\subsubsection{Architecture and protected routes}
All image experiments use the standard-Gaussian VAE interface
\[
q_\phi(z\mid x)=\mathcal N\!\left(u_\phi(x),\operatorname{diag}\sigma_\phi^2(x)\right),
\qquad p(z)=\mathcal N(0,I),
\]
with a convolutional CIFAR-scale encoder/decoder and a $128$-dimensional latent.  The reconstruction branch uses samples $z\sim q_\phi(z\mid x)$; the certificate branch is deterministic and evaluates the posterior mean $u_\phi(x)$ unless a route-local re-instantiation is stated explicitly.  Tiny-ImageNet-200 is resized to $32\times32$ and uses the same CIFAR-scale VAE interface and witness dimension.  The theorem-relevant components are summarized below.

\begin{center}
\small
\setlength{\tabcolsep}{4.5pt}
\renewcommand{\arraystretch}{1.08}
\begin{tabular}{p{0.19\linewidth}p{0.29\linewidth}p{0.42\linewidth}}
\toprule
Component & Map & Role in the reported runs \\
\midrule
Encoder & $x\mapsto (u_\phi(x),\log\sigma_\phi^2(x))$ & Produces the sampled VAE posterior and the deterministic protected mean-code route. \\
Decoder & $z\mapsto \hat x$ & Standard reconstruction path; it is not an input to the certificate. \\
Teacher & $x\mapsto T_x$ & Smoothed PCA--GMM responsibilities in the image experiments; selected before certified training and then frozen. \\
Simplex witness & $u\mapsto \operatorname{softmax}(\beta Vu+\log\bar T)$ & Fixed code-only Student; $\beta=5.0$ in the main image runs. \\
witness-null route & $u=\alpha u_T+Nh_\phi$, $VN=0$ & Keeps the witness-visible component controlled while allowing witness-invisible residual information. \\
\bottomrule
\end{tabular}
\end{center}
The exact layer definitions and run configurations are part of the code/configuration package; the certificate itself depends only on the protected-route separation above, not on a particular convolutional backbone.

\subsubsection{Teacher construction and search}
The image Teachers are constructed from PCA features followed by a diagonal-covariance GMM.  Responsibilities are smoothed to full support, candidate Teachers are evaluated before VAE training, and the selected Teacher is then frozen for every certified run using that Teacher.  The search is designed to trade off an informative but not excessively sharp view.  In the notation of the earlier implementation, the candidate score uses the ingredients
\[
I_T(\omega),\qquad
\frac{K-1}{2K\beta^2}\mathcal E_T(\omega),\qquad
\mathrm{Bal}(\bar T_\omega),\qquad
\mathrm{Fit}(\omega),\qquad
\mathrm{Instab}(\omega),
\]
corresponding respectively to Teacher information, analytic mean-prior energy, component balance, GMM fit, and instability across seeds or mild perturbations.  A representative scalarization is
\[
\mathrm{Score}(\omega)=
\lambda_I I_T(\omega)
-\lambda_E\frac{K-1}{2K\beta^2}\mathcal E_T(\omega)
+\lambda_B\mathrm{Bal}(\bar T_\omega)
+\lambda_F\mathrm{Fit}(\omega)
-\lambda_S\mathrm{Instab}(\omega),
\]
with the concrete search weights treated as implementation choices rather than theorem assumptions.  PCA--GMM is only a reproducible way to instantiate a fixed nontrivial Teacher; the certificate theorem requires the final Teacher/route pair, not this search procedure.

The frozen Teachers used in the five-dataset benchmark have $K=50$ and the following theorem-facing statistics:
\begin{center}
\small
\setlength{\tabcolsep}{5pt}
\begin{tabular}{lccc}
\toprule
Dataset & $K$ & $I_T$ & $\mathcal E_T$ \\
\midrule
CIFAR-10 & 50 & 2.7275 & 53.5416 \\
CIFAR-100 & 50 & 2.7349 & 57.9823 \\
Fashion-MNIST & 50 & 1.8727 & 24.8682 \\
SVHN & 50 & 2.7538 & 53.0375 \\
Tiny-ImageNet-200 & 50 & 2.6080 & 62.1318 \\
\bottomrule
\end{tabular}
\end{center}
For Tiny-ImageNet-200, the frozen final Teacher uses 64 PCA dimensions, a diagonal GMM, temperature $0.7$, probability smoothing $\epsilon=0.1$, and GMM seed 1. The remaining per-dataset search-grid metadata are preserved in the frozen Teacher artifacts/configuration package; these search choices are implementation details rather than theorem assumptions.

\subsubsection{Training and reporting protocol}
Unless a control explicitly changes one of these quantities, the main image runs use 80 training epochs, $d_z=128$, witness gain $\beta=5.0$, $\beta_{\mathrm{KL}}=1.0$, and $\lambda_{\mathrm{TS}}=1.0$.  Analytic-prefit variants use five prefit epochs.  The KL-pressure stress test sets $\beta_{\mathrm{KL}}=4.0$.  Main five-dataset summaries use five random seeds; the KL-pressure and counterfactual-control blocks use three seeds.  Multi-seed tables report the final training epoch and sample standard deviation. For the reported image benchmark tables, $\bar T$, $I_T$, and $L_{\mathrm{TS}}$ are computed self-consistently on the same empirical training split on which the frozen Teacher responsibilities are stored. Accordingly, the reported $G_T$ values are exact certificates for that empirical training distribution, not claims of held-out population generalization; any separate held-out diagnostic is labeled as such. Reconstruction, KL, active units, effective rank, MMD/prior checks where available, and visible-path residuals are diagnostics; $G_T=I_T-L_{\mathrm{TS}}$ remains the certificate quantity.

\subsubsection{Gradient-dominance diagnostics}
To distinguish the function-space escape field from its realized parameter-space update, define
\[
g_{\mathrm{TS}}=\nabla_\phi L_{\mathrm{TS}},\qquad
 g_{\mathrm{VAE}}=\nabla_\phi(\mathcal L_{\mathrm{rec}}+\beta_{\mathrm{KL}}\mathcal L_{\mathrm{KL}}),
\]
and monitor
\[
\Gamma_{\mathrm{grad}}=
\frac{\lambda_{\mathrm{TS}}\|g_{\mathrm{TS}}\|}{\|g_{\mathrm{VAE}}\|+10^{-12}},
\qquad
\cos(g_{\mathrm{TS}},g_{\mathrm{VAE}}).
\]
The ratio measures relative scale, while the cosine diagnoses agreement or opposition between the two parameter-space gradients.  Early CIFAR-10 runs give mean $\Gamma_{\mathrm{grad}}$ values of about $0.30$--$0.44$ across seeds, with many collapse-neighborhood steps above one.  The cosine is used as a conflict diagnostic: a strongly negative value warns that the VAE objective opposes the witness update, while a non-opposing value indicates that the two signals are not cancelling by direction.  These measurements support the local optimization interpretation only; they are neither part of the certificate nor an optimizer-independent convergence theorem.

\subsection{Complete diagnostic tables}
\label{app:complete-diagnostic-tables}

The main text uses slim summaries so that the certificate message is readable.  This appendix restores the diagnostic columns omitted from the main CIFAR, five-dataset, KL-pressure, and counterfactual results.  These diagnostics are not themselves the certificate.  They are included to make the experimental record complete and auditable.

\begin{table}[htbp]
\centering
\scriptsize
\setlength{\tabcolsep}{3.0pt}
\renewcommand{\arraystretch}{1.08}
\caption{Complete diagnostics for the main CIFAR-10 and CIFAR-100 runs.}
\label{tab:app-full-main-cifar}
\resizebox{\linewidth}{!}{%
\begin{tabular}{llrccccccc}
\toprule
Dataset & Method & Seeds & $G_T\uparrow$ & $L_{\mathrm{TS}}\downarrow$ & Rec$\downarrow$ & KL & Active units & Eff.\ rank & Path residual$\downarrow$ \\
\midrule
CIFAR-10 & VAE & 5 & $-1.987 \pm 0.360$ & $4.714 \pm 0.360$ & $50.547 \pm 0.675$ & $22.922 \pm 0.911$ & $50.8 \pm 6.6$ & $33.19 \pm 0.31$ & $205.974 \pm 39.474$ \\
CIFAR-10 & RST & 5 & $1.885 \pm 0.033$ & $0.842 \pm 0.033$ & $50.710 \pm 0.450$ & $23.048 \pm 0.406$ & $94.4 \pm 3.6$ & $38.06 \pm 0.14$ & $45.505 \pm 9.327$ \\
CIFAR-10 & RST-prefit & 5 & $1.914 \pm 0.029$ & $0.814 \pm 0.029$ & $50.794 \pm 0.313$ & $22.818 \pm 0.299$ & $90.6 \pm 3.6$ & $38.58 \pm 0.15$ & $68.555 \pm 9.314$ \\
CIFAR-10 & RST-alpha-prefit & 5 & $1.841 \pm 0.062$ & $0.886 \pm 0.062$ & $50.416 \pm 0.495$ & $23.342 \pm 0.480$ & $95.0 \pm 4.2$ & $38.15 \pm 0.25$ & $65.776 \pm 29.899$ \\
CIFAR-10 & RST-alpha-logit-prefit & 5 & $1.836 \pm 0.041$ & $0.891 \pm 0.041$ & $50.511 \pm 0.514$ & $23.324 \pm 0.506$ & $91.4 \pm 2.3$ & $38.33 \pm 0.21$ & $75.930 \pm 19.893$ \\
CIFAR-10 & RST-nullspace-route & 5 & $0.276 \pm 0.000$ & $2.452 \pm 0.000$ & $50.191 \pm 0.294$ & $23.043 \pm 0.288$ & $44.6 \pm 3.3$ & $33.06 \pm 0.18$ & $0.001 \pm 0.000$ \\
CIFAR-100 & VAE & 5 & $-2.074 \pm 0.242$ & $4.809 \pm 0.242$ & $50.080 \pm 0.507$ & $23.202 \pm 0.497$ & $50.6 \pm 3.8$ & $32.16 \pm 0.25$ & $215.797 \pm 48.189$ \\
CIFAR-100 & RST & 5 & $1.556 \pm 0.021$ & $1.179 \pm 0.021$ & $49.892 \pm 0.248$ & $23.784 \pm 0.255$ & $92.8 \pm 4.0$ & $38.16 \pm 0.11$ & $42.849 \pm 8.226$ \\
CIFAR-100 & RST-prefit & 5 & $1.692 \pm 0.038$ & $1.043 \pm 0.038$ & $50.242 \pm 0.736$ & $23.448 \pm 0.396$ & $89.8 \pm 3.3$ & $38.42 \pm 0.28$ & $36.692 \pm 10.508$ \\
CIFAR-100 & RST-alpha-prefit & 5 & $1.488 \pm 0.056$ & $1.247 \pm 0.056$ & $50.013 \pm 0.319$ & $23.564 \pm 0.268$ & $100.4 \pm 4.3$ & $37.74 \pm 0.23$ & $37.024 \pm 9.471$ \\
CIFAR-100 & RST-alpha-logit-prefit & 5 & $1.542 \pm 0.039$ & $1.193 \pm 0.039$ & $50.056 \pm 0.412$ & $23.547 \pm 0.520$ & $94.2 \pm 2.6$ & $38.02 \pm 0.28$ & $36.171 \pm 6.450$ \\
CIFAR-100 & RST-nullspace-route & 5 & $0.102 \pm 0.000$ & $2.633 \pm 0.000$ & $49.768 \pm 0.450$ & $23.361 \pm 0.368$ & $49.0 \pm 2.3$ & $31.94 \pm 0.24$ & $0.001 \pm 0.000$ \\
\bottomrule
\end{tabular}%
}
\end{table}

\begin{table}[htbp]
\centering
\scriptsize
\setlength{\tabcolsep}{3.0pt}
\renewcommand{\arraystretch}{1.08}
\caption{Complete diagnostics for the five-dataset extension corresponding to Table~\ref{main:tab-five}.}
\label{tab:app-full-five-datasets}
\resizebox{\linewidth}{!}{%
\begin{tabular}{llrccccccc}
\toprule
Dataset & Method & Seeds & $G_T\uparrow$ & $L_{\mathrm{TS}}\downarrow$ & Rec$\downarrow$ & KL & Active units & Eff.\ rank & Path residual$\downarrow$ \\
\midrule
CIFAR-10 & VAE & 5 & $-1.987 \pm 0.360$ & $4.714 \pm 0.360$ & $50.547 \pm 0.675$ & $22.922 \pm 0.911$ & $50.800 \pm 6.573$ & $33.186 \pm 0.307$ & $205.974 \pm 39.474$ \\
CIFAR-10 & RST-alpha-prefit & 5 & $1.841 \pm 0.062$ & $0.886 \pm 0.062$ & $50.416 \pm 0.495$ & $23.342 \pm 0.480$ & $95.000 \pm 4.183$ & $38.148 \pm 0.254$ & $65.776 \pm 29.899$ \\
CIFAR-10 & RST-alpha-logit-prefit & 5 & $1.836 \pm 0.041$ & $0.891 \pm 0.041$ & $50.511 \pm 0.514$ & $23.324 \pm 0.506$ & $91.400 \pm 2.302$ & $38.328 \pm 0.211$ & $75.930 \pm 19.893$ \\
CIFAR-10 & RST-nullspace-route & 5 & $0.276 \pm 0.000$ & $2.452 \pm 0.000$ & $50.191 \pm 0.294$ & $23.043 \pm 0.288$ & $44.600 \pm 3.286$ & $33.060 \pm 0.175$ & $0.001 \pm 0.000$ \\
CIFAR-100 & VAE & 5 & $-2.074 \pm 0.242$ & $4.809 \pm 0.242$ & $50.080 \pm 0.507$ & $23.202 \pm 0.497$ & $50.600 \pm 3.847$ & $32.162 \pm 0.253$ & $215.797 \pm 48.189$ \\
CIFAR-100 & RST-alpha-prefit & 5 & $1.488 \pm 0.056$ & $1.247 \pm 0.056$ & $50.013 \pm 0.319$ & $23.564 \pm 0.268$ & $100.400 \pm 4.336$ & $37.742 \pm 0.229$ & $37.024 \pm 9.471$ \\
CIFAR-100 & RST-alpha-logit-prefit & 5 & $1.542 \pm 0.039$ & $1.193 \pm 0.039$ & $50.056 \pm 0.412$ & $23.547 \pm 0.520$ & $94.200 \pm 2.588$ & $38.021 \pm 0.277$ & $36.171 \pm 6.450$ \\
CIFAR-100 & RST-nullspace-route & 5 & $0.102 \pm 0.000$ & $2.633 \pm 0.000$ & $49.768 \pm 0.450$ & $23.361 \pm 0.368$ & $49.000 \pm 2.345$ & $31.938 \pm 0.243$ & $0.001 \pm 0.000$ \\
Fashion-MNIST & VAE & 5 & $-0.277 \pm 0.328$ & $2.150 \pm 0.328$ & $220.245 \pm 0.190$ & $15.546 \pm 0.131$ & $15.000 \pm 2.000$ & $12.577 \pm 0.301$ & $6.018 \pm 7.715$ \\
Fashion-MNIST & RST-alpha-prefit & 5 & $1.503 \pm 0.060$ & $0.370 \pm 0.060$ & $220.833 \pm 0.391$ & $15.879 \pm 0.215$ & $19.800 \pm 0.447$ & $13.958 \pm 0.251$ & $100.903 \pm 20.845$ \\
Fashion-MNIST & RST-alpha-logit-prefit & 5 & $1.626 \pm 0.054$ & $0.247 \pm 0.054$ & $220.474 \pm 0.167$ & $15.830 \pm 0.264$ & $22.200 \pm 2.168$ & $14.411 \pm 0.199$ & $46.225 \pm 24.271$ \\
Fashion-MNIST & RST-nullspace-route & 5 & $0.103 \pm 0.000$ & $1.770 \pm 0.000$ & $220.466 \pm 0.267$ & $15.440 \pm 0.232$ & $13.600 \pm 0.894$ & $12.589 \pm 0.037$ & $0.000 \pm 0.000$ \\
SVHN & VAE & 5 & $-1.088 \pm 0.133$ & $3.842 \pm 0.133$ & $22.501 \pm 0.388$ & $11.598 \pm 0.264$ & $23.000 \pm 2.449$ & $18.198 \pm 0.173$ & $109.389 \pm 19.873$ \\
SVHN & RST-alpha-prefit & 5 & $1.943 \pm 0.026$ & $0.811 \pm 0.026$ & $22.326 \pm 0.360$ & $12.219 \pm 0.313$ & $68.600 \pm 2.702$ & $24.928 \pm 0.120$ & $67.584 \pm 6.953$ \\
SVHN & RST-alpha-logit-prefit & 5 & $1.967 \pm 0.029$ & $0.786 \pm 0.029$ & $22.292 \pm 0.225$ & $12.347 \pm 0.120$ & $68.200 \pm 1.483$ & $25.262 \pm 0.321$ & $69.967 \pm 7.972$ \\
SVHN & RST-nullspace-route & 5 & $0.106 \pm 0.000$ & $2.648 \pm 0.000$ & $22.189 \pm 0.300$ & $11.904 \pm 0.302$ & $23.400 \pm 1.517$ & $17.984 \pm 0.119$ & $0.002 \pm 0.000$ \\
Tiny-ImageNet-200 & VAE & 5 & $-2.279 \pm 0.360$ & $4.887 \pm 0.360$ & $50.341 \pm 0.445$ & $23.911 \pm 0.505$ & $50.600 \pm 2.408$ & $35.213 \pm 0.264$ & $245.728 \pm 45.396$ \\
Tiny-ImageNet-200 & RST-alpha-prefit & 5 & $1.321 \pm 0.036$ & $1.287 \pm 0.036$ & $50.318 \pm 0.507$ & $24.274 \pm 0.454$ & $100.000 \pm 3.000$ & $41.309 \pm 0.530$ & $33.359 \pm 1.538$ \\
Tiny-ImageNet-200 & RST-alpha-logit-prefit & 5 & $1.317 \pm 0.101$ & $1.291 \pm 0.101$ & $50.406 \pm 0.346$ & $24.159 \pm 0.375$ & $98.400 \pm 1.517$ & $41.736 \pm 0.280$ & $42.920 \pm 14.856$ \\
Tiny-ImageNet-200 & RST-nullspace-route & 5 & $0.102 \pm 0.000$ & $2.506 \pm 0.000$ & $50.113 \pm 0.498$ & $24.059 \pm 0.508$ & $46.000 \pm 1.871$ & $35.215 \pm 0.086$ & $0.001 \pm 0.000$ \\
\bottomrule
\end{tabular}%
}
\end{table}

\begin{table}[htbp]
\centering
\scriptsize
\setlength{\tabcolsep}{3.2pt}
\renewcommand{\arraystretch}{1.08}
\caption{Complete diagnostics for the KL-pressure stress test with $\beta_{\mathrm{KL}}=4.0$.}
\label{tab:app-full-stress}
\resizebox{\linewidth}{!}{%
\begin{tabular}{llrccccccc}
\toprule
Dataset & Method & Seeds & $G_T\uparrow$ & $L_{\mathrm{TS}}\downarrow$ & Rec$\downarrow$ & KL & Active units & Eff.\ rank & Path residual$\downarrow$ \\
\midrule
CIFAR-10 & VAE & 3 & $-0.637 \pm 0.274$ & $3.364 \pm 0.274$ & $82.668 \pm 0.801$ & $6.980 \pm 0.201$ & $13.333 \pm 1.528$ & $11.092 \pm 0.077$ & $53.699 \pm 22.029$ \\
CIFAR-10 & RST-alpha-prefit & 3 & $1.243 \pm 0.024$ & $1.484 \pm 0.024$ & $82.174 \pm 1.005$ & $7.190 \pm 0.254$ & $53.667 \pm 1.528$ & $14.183 \pm 0.223$ & $38.268 \pm 9.246$ \\
CIFAR-10 & RST-alpha-logit-prefit & 3 & $1.317 \pm 0.004$ & $1.410 \pm 0.004$ & $81.714 \pm 0.954$ & $7.311 \pm 0.212$ & $55.000 \pm 1.000$ & $14.162 \pm 0.099$ & $33.969 \pm 0.879$ \\
CIFAR-10 & RST-nullspace-route & 3 & $0.159 \pm 0.100$ & $2.568 \pm 0.100$ & $82.352 \pm 0.492$ & $7.008 \pm 0.096$ & $11.000 \pm 0.000$ & $11.109 \pm 0.067$ & $0.000 \pm 0.000$ \\
CIFAR-100 & VAE & 3 & $-0.547 \pm 0.051$ & $3.282 \pm 0.051$ & $81.899 \pm 0.319$ & $7.580 \pm 0.105$ & $11.000 \pm 1.000$ & $11.137 \pm 0.027$ & $51.976 \pm 9.911$ \\
CIFAR-100 & RST-alpha-prefit & 3 & $0.770 \pm 0.108$ & $1.965 \pm 0.108$ & $81.867 \pm 1.265$ & $7.686 \pm 0.287$ & $36.333 \pm 1.528$ & $13.326 \pm 0.210$ & $24.504 \pm 19.293$ \\
CIFAR-100 & RST-alpha-logit-prefit & 3 & $0.853 \pm 0.013$ & $1.882 \pm 0.013$ & $82.298 \pm 0.503$ & $7.545 \pm 0.089$ & $37.000 \pm 6.928$ & $13.233 \pm 0.129$ & $13.219 \pm 0.644$ \\
CIFAR-100 & RST-nullspace-route & 3 & $0.101 \pm 0.000$ & $2.634 \pm 0.000$ & $82.246 \pm 0.526$ & $7.516 \pm 0.125$ & $14.000 \pm 2.646$ & $11.000 \pm 0.115$ & $0.000 \pm 0.000$ \\
\bottomrule
\end{tabular}%
}
\end{table}

\begin{table}[htbp]
\centering
\scriptsize
\setlength{\tabcolsep}{3.2pt}
\renewcommand{\arraystretch}{1.08}
\caption{Complete diagnostics for baseline and teacher counterfactual controls.}
\label{tab:app-full-counterfactual-controls}
\resizebox{\linewidth}{!}{%
\begin{tabular}{lll r cccccc}
\toprule
Dataset & Block & Condition & Seeds & $G_T\uparrow$ & $L_{\mathrm{TS}}\downarrow$ & Rec$\downarrow$ & KL & Active units & Eff.\ rank \\
\midrule
CIFAR-10 & baseline & VAE+KL-warmup & 3 & $-1.628 \pm 0.096$ & $4.355 \pm 0.096$ & $50.135 \pm 0.534$ & $23.053 \pm 0.492$ & $42.333 \pm 0.577$ & $34.733 \pm 0.458$ \\
CIFAR-10 & baseline & VAE+cyclical-annealing & 3 & $-1.962 \pm 0.137$ & $4.690 \pm 0.137$ & $49.909 \pm 0.402$ & $23.622 \pm 0.240$ & $50.667 \pm 1.155$ & $37.456 \pm 0.440$ \\
CIFAR-10 & baseline & VAE+free-bits & 3 & $-2.256 \pm 0.260$ & $4.983 \pm 0.260$ & $53.151 \pm 1.001$ & $22.205 \pm 0.752$ & $128.000 \pm 0.000$ & $36.338 \pm 0.890$ \\
CIFAR-10 & teacher & degraded\_rho0.5 & 3 & $0.540 \pm 0.005$ & $0.384 \pm 0.005$ & $50.154 \pm 0.225$ & $23.012 \pm 0.244$ & $50.000 \pm 3.000$ & $35.067 \pm 0.044$ \\
CIFAR-10 & teacher & degraded\_rho0.75 & 3 & $0.147 \pm 0.001$ & $0.177 \pm 0.001$ & $50.158 \pm 0.638$ & $23.024 \pm 0.612$ & $38.667 \pm 2.517$ & $34.016 \pm 0.285$ \\
CIFAR-10 & teacher & random\_dirichlet & 3 & $-0.004 \pm 0.001$ & $0.372 \pm 0.001$ & $50.238 \pm 0.423$ & $23.137 \pm 0.509$ & $45.000 \pm 4.359$ & $33.112 \pm 0.075$ \\
CIFAR-10 & teacher & shuffled & 3 & $0.159 \pm 0.022$ & $2.569 \pm 0.022$ & $50.131 \pm 0.167$ & $23.200 \pm 0.260$ & $42.000 \pm 1.732$ & $34.256 \pm 0.273$ \\
CIFAR-100 & baseline & VAE+KL-warmup & 3 & $-2.227 \pm 0.483$ & $4.961 \pm 0.483$ & $49.255 \pm 0.039$ & $23.800 \pm 0.289$ & $41.667 \pm 2.082$ & $33.868 \pm 0.184$ \\
CIFAR-100 & baseline & VAE+cyclical-annealing & 3 & $-2.218 \pm 0.273$ & $4.953 \pm 0.273$ & $49.235 \pm 0.799$ & $23.985 \pm 0.806$ & $51.667 \pm 0.577$ & $36.900 \pm 0.240$ \\
CIFAR-100 & baseline & VAE+free-bits & 3 & $-2.174 \pm 0.099$ & $4.909 \pm 0.099$ & $53.343 \pm 0.268$ & $22.127 \pm 0.291$ & $128.000 \pm 0.000$ & $34.406 \pm 0.279$ \\
CIFAR-100 & teacher & degraded\_rho0.5 & 3 & $0.409 \pm 0.004$ & $0.513 \pm 0.004$ & $50.230 \pm 0.171$ & $22.987 \pm 0.115$ & $51.000 \pm 5.000$ & $34.190 \pm 0.131$ \\
CIFAR-100 & teacher & degraded\_rho0.75 & 3 & $0.102 \pm 0.004$ & $0.221 \pm 0.004$ & $49.971 \pm 0.502$ & $23.304 \pm 0.353$ & $47.333 \pm 3.215$ & $32.846 \pm 0.206$ \\
CIFAR-100 & teacher & random\_dirichlet & 3 & $-0.005 \pm 0.000$ & $0.373 \pm 0.000$ & $50.568 \pm 0.528$ & $22.813 \pm 0.443$ & $50.000 \pm 4.583$ & $31.547 \pm 0.414$ \\
CIFAR-100 & teacher & shuffled & 3 & $0.138 \pm 0.011$ & $2.597 \pm 0.011$ & $50.007 \pm 0.368$ & $23.395 \pm 0.588$ & $53.000 \pm 4.359$ & $32.840 \pm 0.218$ \\
\bottomrule
\end{tabular}%
}
\end{table}

\FloatBarrier
\section{Mechanistic and numerical audits}
\label{app:mechanistic-audits}
\mainback{sec:main-theory}{Section~\ref*{sec:main-theory} (theory and optimization diagnostics)}
This block collects implementation-facing checks that connect the analytic objects to optimization and diagnostics: gradient flow, the closed-form path audit, code sanity, failed-certificate diagnosis, decoder-side diagnostics, escape, and witness-weight behavior.  These audits support interpretation; they do not enlarge the theorem claims.

\subsection{Gradient-flow schematic}
\label{app:gradient-flow-schematic}
Figure~\ref{fig:gradient-flow} separates forward computation from optimization signals.  The Teacher--Student loss compares the frozen Teacher target with the fixed code-only witness and sends its gradient only through the encoder-side mean-code route.  Reconstruction trains the encoder and decoder, while KL regularizes the posterior parameters.  This figure is explanatory rather than an additional empirical claim; the parameter-space caveat and gradient diagnostics are stated in the main text and later in the appendix.
\begin{figure}[htbp]
\centering
\resizebox{0.94\linewidth}{!}{%
\begin{tikzpicture}[
  x=1cm,y=1cm,
  font=\small,
  >=Latex,
  box/.style={draw, rounded corners=3pt, align=center, inner xsep=5pt, inner ysep=5pt, minimum height=0.82cm, fill=white},
  widebox/.style={box, minimum width=3.45cm},
  smallbox/.style={box, minimum width=1.05cm},
  teacherbox/.style={box, dashed, dash pattern=on 4pt off 3pt, minimum width=2.35cm},
  arr/.style={->, line width=0.55pt},
  gradheavy/.style={->, line width=1.85pt},
  every node/.style={text=black}
]
% nodes -- unchanged from v36 except Recon. loss
\node[teacherbox] (teacher) at (0,3.35) {Fixed teacher\\$T_x$};
\node[widebox, minimum width=3.85cm] (align) at (5.1,3.65) {Alignment\\$\mathcal{L}_{TS}=\mathrm{KL}\!\left(T_x\,\Vert\,S(u_\phi(x))\right)$};
\node[widebox, minimum width=4.05cm] (witness) at (5.1,2.05) {Fixed witness\\$S(u_\phi(x))$\\$=\operatorname{softmax}\!\left(\beta V u_\phi(x)+\log\bar T\right)$};
\node[smallbox] (x) at (0,0) {$x$};
\node[widebox, minimum width=1.75cm] (enc) at (2.45,0) {Encoder\\$q_\phi(z\mid x)$};
\node[widebox, minimum width=2.25cm] (mean) at (5.1,0) {Mean code\\$u_\phi(x)$};
\node[widebox, minimum width=2.25cm] (logv) at (5.1,-1.55) {Log variance\\$\log\sigma_\phi^2(x)$};
\node[widebox, minimum width=1.85cm] (sample) at (8.15,0) {Sample\\$z\sim q_\phi$};
\node[widebox, minimum width=1.85cm] (dec) at (10.75,0) {Decoder\\$p_\theta(x\mid z)$};
\node[smallbox] (xhat) at (12.95,0) {$\hat x$};
\node[widebox, minimum width=2.55cm] (kl) at (5.1,-3.0) {KL to $\mathcal{N}(0,I)$};
\node[widebox, minimum width=2.7cm, minimum height=1.0cm] (recon) at (5.1,-4.52)
{Recon. loss};

% forward computation -- unchanged away from Recon. loss
\draw[arr] (x.east) -- (enc.west);
\draw[arr] (x.north) -- (teacher.south);
\draw[arr] (teacher.east) -- (align.west);
\draw[arr] (witness.north) -- (align.south);
\draw[arr] (mean.north) -- (witness.south);
\draw[arr] (enc.east) -- (mean.west);
\draw[arr] (enc.south east) -- (logv.west);
\draw[arr] (mean.east) -- (sample.west);
\draw[arr] (logv.east) -- (sample.south west);
\draw[arr] (sample.east) -- (dec.west);
\draw[arr] (dec.east) -- (xhat.west);
\draw[arr] (mean.south east) .. controls +(1.35,-0.55) and +(1.20,0.70) .. (kl.north east);
\draw[arr] (logv.south) -- (kl.north);

% x and xhat feed Recon. loss
\coordinate (xlow)    at (0,-4.42);
\coordinate (xhatlow) at (12.95,-4.42);
\draw[arr] (x.south)    -- (xlow)    -- (recon.west);
\draw[arr] (xhat.south) -- (xhatlow) -- (recon.east);

% gradients
\draw[gradheavy] (align.south west) .. controls +(-1.35,-0.8) and +(0.2,1.25) .. (enc.north);

% KL -> Encoder: rerouted to avoid overlap with Recon. loss -> Encoder.
% First control point is relative to kl.west; second is relative to the endpoint.
\coordinate (encSouthRight) at ($(enc.south)+(0.28,0)$);
\draw[gradheavy] (kl.west)
  .. controls +(-0.95,-0.10) and +(0.70,-1.05) .. (encSouthRight);

% Recon. loss -> Encoder / Decoder: unchanged from the user's preferred layout.
\coordinate (encSouthLeft) at ($(enc.south)+(-0.22,0)$);
\draw[gradheavy] (recon.north west)
  .. controls +(-1.10,0.75) and +(0.45,-1.15) .. (encSouthLeft);
\draw[gradheavy] (recon.north east)
  .. controls +(0.55,1.25) and +(0.05,-2.25) .. (dec.south);
\end{tikzpicture}%
}
\caption{Gradient-flow view of RST. Thin arrows show forward computation; thick arrows show optimization signals. The fixed Teacher and witness prevent the certified-stage alignment loss from being reduced by Teacher--Student co-adaptation.  The diagram is restored to the v15 layout.}
\label{fig:gradient-flow}
\end{figure}

\subsection{Analytic margin--energy path audit: full CIFAR-100 values}
\label{app:analytic-margin-energy-audit}
\mainback{sec:main-theory}{Section~\ref*{sec:main-theory} and Figure~\ref*{fig:analytic-margin-energy-audit}}
This appendix records every sampled point used in Figure~\ref{fig:analytic-margin-energy-audit}.  It is a deterministic evaluation of the closed-form simplex construction on the 10,000-image CIFAR-100 test split, with $\epsilon=0.02$, $\beta=5$, and $d_z=192$.  ``Slack'' denotes $G(\alpha)-\alpha I_T$.  The $-1.49\times10^{-8}$ coarse endpoint slack is floating-point roundoff at a theoretical equality.

\begin{table}[htbp]
\centering
\scriptsize
\setlength{\tabcolsep}{5pt}
\caption{Coarse Teacher ($K=20$): $I_T=2.8456666$, $E(1)=0.8565164$.}
\begin{tabular}{rrrrr}
\toprule
$\alpha$ & $G(\alpha)$ & $\alpha I_T$ & slack & $E(\alpha)$\\
\midrule
0.00 & -0.0000000 & 0.0000000 & -0.0000000 & 0.0000000\\
0.05 & 0.3175338 & 0.1422833 & 0.1752505 & 0.0021413\\
0.10 & 0.6273866 & 0.2845667 & 0.3428199 & 0.0085652\\
0.20 & 1.2132628 & 0.5691333 & 0.6441295 & 0.0342607\\
0.40 & 2.1512680 & 1.1382666 & 1.0130013 & 0.1370426\\
0.60 & 2.6512941 & 1.7074000 & 0.9438941 & 0.3083459\\
0.80 & 2.8170229 & 2.2765333 & 0.5404896 & 0.5481705\\
1.00 & 2.8456666 & 2.8456666 & -0.0000000 & 0.8565164\\
\bottomrule
\end{tabular}
\end{table}

\begin{table}[htbp]
\centering
\scriptsize
\setlength{\tabcolsep}{5pt}
\caption{Fine Teacher ($K=100$): $I_T=4.4169271$, $E(1)=1.4153098$.}
\begin{tabular}{rrrrr}
\toprule
$\alpha$ & $G(\alpha)$ & $\alpha I_T$ & slack & $E(\alpha)$\\
\midrule
0.00 & 0.0000000 & 0.0000000 & 0.0000000 & 0.0000000\\
0.05 & 0.4111677 & 0.2208464 & 0.1903214 & 0.0035383\\
0.10 & 0.8195940 & 0.4416927 & 0.3779013 & 0.0141531\\
0.20 & 1.6220519 & 0.8833854 & 0.7386665 & 0.0566124\\
0.40 & 3.0774768 & 1.7667708 & 1.3107059 & 0.2264496\\
0.60 & 4.0313523 & 2.6501562 & 1.3811960 & 0.5095115\\
0.80 & 4.3657539 & 3.5335417 & 0.8322122 & 0.9057983\\
1.00 & 4.4169271 & 4.4169271 & 0.0000000 & 1.4153098\\
\bottomrule
\end{tabular}
\end{table}

The numerical audits are: inverse maximum error $6.966\times10^{-9}$ (coarse) and $7.688\times10^{-9}$ (fine); maximum geometric-path error $9.992\times10^{-16}$ and $1.332\times10^{-15}$; maximum quadratic-energy-law error $1.110\times10^{-16}$ and $0$; minimum lower-bound slack $-1.490\times10^{-8}$ and $7.843\times10^{-9}$.  These numbers audit implementation of exact identities and should not be interpreted as statistical confidence intervals or learned-model generalization results.

\subsection{Teacher-code encoder sanity check}
\label{app:teacher-code-sanity}

The teacher-code interpretation can be checked without a full VAE by training a small encoder to predict the analytic teacher code from vector inputs.  This diagnostic is not used as a main empirical claim; it records whether the map \(x\mapsto\uT(x)\) is learnable only when the teacher is aligned with the input.

In a synthetic vector replay setting, we constructed a five-component risk teacher from an input feature and trained an encoder to predict the corresponding teacher code.  On a held-out validation split, the aligned teacher achieved
\[
I_T=1.429147,\qquad
L_{\mathrm{TS}}=0.153129,\qquad
G_T=1.276019 .
\]
We then replaced the teacher by two negative controls.  A random-label teacher gave
\[
I_T=1.599293,\qquad
L_{\mathrm{TS}}=1.690099,\qquad
G_T=-0.090807 ,
\]
and a shuffled version of the original risk teacher gave
\[
I_T=1.429147,\qquad
L_{\mathrm{TS}}=1.536108,\qquad
G_T=-0.106961 .
\]
Thus the positive held-out margin appears for the input-aligned teacher but not for random or sample-shuffled teachers.  This supports the intended interpretation of teacher-code learning: the encoder can learn a readable teacher-code representation when the teacher view is a genuine function of the input, while arbitrary or misaligned pseudo-teachers fail the held-out certificate.

\subsection{Non-positive margins, route feasibility, and decoder-side diagnostics}
\label{app:nonpositive-teacher-path}
\mainback{fig:rst-main-architecture}{Figure~\ref*{fig:rst-main-architecture} (route and decoder diagnostics)}

The boundary case
\[
L_{\mathrm{TS}}\ge I_T,\qquad \uphi(x)\not\equiv \mathrm{constant}
\]
is nonconstant but not Teacher-certified on the prescribed route.  It may encode other information, or Teacher information in coordinates unreadable by the fixed witness.

\subsubsection{Three-point geometry of the origin, an intermediate code, and the Teacher code}

Let
\[
S(u)=\softmax(\log\bar T+\beta Vu),
\qquad
W=\beta V,
\]
and let \(\uT(x)\) be an analytic teacher deterministic code satisfying \(S(\uT(x))=T_x\). The three relevant points are the constant baseline \(0\), the current deterministic code \(u(x)=\uphi(x)\), and the teacher-perfect deterministic code \(\uT(x)\). With
\[
a_u(x)=\log\bar T+Wu(x),\qquad
a_T(x)=\log\bar T+W\uT(x),
\]
and \(A(a)=\log\sum_k e^{a_k}\), the log-sum-exp Bregman divergence is
\[
D_A(a_u,a_T)=A(a_u)-A(a_T)-\nabla A(a_T)^\top(a_u-a_T).
\]
Since \(\nabla A(a_T)=T_x\), we have the exact identity
\[
D_A(a_u,a_T)=\KL(T_x\|S(u(x))).
\]
Therefore
\[
\E_x D_A(a_u,a_T)=L_{\mathrm{TS}}(u),
\qquad
\E_x D_A(a_0,a_T)=I_T,
\]
where \(a_0=\log\bar T\). Consequently,
\[
L_{\mathrm{TS}}(u)>I_T
\quad\Longleftrightarrow\quad
\E_x D_A(a_u,a_T)>\E_x D_A(a_0,a_T).
\]
Thus the unfavorable regime means that \(u\) is, on average, farther from the teacher deterministic code than the constant baseline in the witness-Bregman geometry.

There is also a useful three-point decomposition:
\[
L_{\mathrm{TS}}(u)-I_T
=
\underbrace{\E_x\KL(\bar T\|S(u(x)))}_{\text{departure cost from }0}
-
\underbrace{\E_x\langle T_x-\bar T,Wu(x)\rangle}_{\text{teacher-alignment gain}}.
\]
This identity explains why \(\uphi\neq\mathrm{constant}\) is not enough. The code may leave the origin, but if its teacher-alignment gain is smaller than the cost of leaving the constant baseline, then \(L_{\mathrm{TS}}\ge I_T\).

If one replaces KL by Jensen--Shannon divergence, then the symmetric distance \(d_{\mathrm{JS}}(P,Q)=\sqrt{\mathrm{JS}(P,Q)}\) gives an ordinary metric geometry among \(T_x\), \(\bar T\), and \(S(u)\). It can be used as a diagnostic: one asks whether \(d_{\mathrm{JS}}(T_x,S(u))<d_{\mathrm{JS}}(T_x,\bar T)\). However, the KL/log-sum-exp Bregman form is more useful for the present theory because it gives the exact alignment-gain decomposition and the escape gradient.

\subsubsection{Route-feasibility diagnostics}
\label{app:route-feasibility}
Two diagnostics separate a failed prescribed certificate from a broader representation failure.  First, freeze the encoder and fit a flexible probe $P_\psi(c\mid \uphi(x))$ with
\[
L_{\mathrm{probe}}=\E_x\KL\!\left(T_x\|P_\psi(\cdot\mid\uphi(x))\right).
\]
If $L_{\mathrm{probe}}<I_T$ while $G_T\le0$, the code contains Teacher-view information that is readable by another route, so the failure is specific to the fixed witness.  Second, freeze the Teacher, witness geometry, and decoder and optimize only the encoder-side witness route,
\[
\min_\phi\;L_{\mathrm{TS}}=\min_\phi\;\E_x\KL\!\left(T_x\|S(\uphi(x))\right).
\]
If this crosses $I_T$, the route is representationally feasible and the original failure is more plausibly due to objective conflict or insufficient alignment.  The frozen-probe experiments in Appendix~\ref{app:functional-interface-full} instantiate the first case directly: both Aux controls have $G_{\rm probe}>0$ while $G_{\rm pre}<0$ on both views, and cross-seed and cross-architecture adaptation tests further localize the mismatch to producer--consumer coordinates.  If neither diagnostic beats $I_T$, one should inspect Teacher sharpness/stability, witness gain, latent or block dimension, and encoder capacity.  Thus $G_T\le0$ is best treated as a failed sufficient certificate plus a route-feasibility diagnostic, not as a collapse verdict.

\subsubsection{Safe Teacher paths and the role of the path scale}

For \(0\le\alpha\le1\), the path
\[
u_\alpha(x)=\alpha \uT(x)
\]
moves from the constant baseline to the teacher-perfect deterministic code. Along this path,
\[
L_{\mathrm{TS}}(0)=I_T,\qquad L_{\mathrm{TS}}(\uT)=0,
\]
and convexity in logits gives
\[
L_{\mathrm{TS}}(\alpha \uT)\le (1-\alpha)I_T.
\]
Thus any sufficiently positive \(\alpha\) gives a positive certificate margin. For a target margin $0<\tau<I_T$, one can choose the fixed safety scale
\[
\alpha_\star(\tau)=\inf\{\alpha\in[0,1]:\; L_{\mathrm{TS}}(\alpha \uT)\le I_T-\tau\},
\]
subject also to a prior-cost budget such as
\[
\frac12\alpha^2\E_x\|\uT(x)\|^2\le\kappa.
\]
This makes \(\alpha\) a fixed analytic design parameter or a prefit scale, not a free parameter that is minimized jointly with the VAE objective.

If the trained code is decomposed as
\[
\uphi(x)=\alpha \uT(x)+r(x),
\]
the relevant off-path residual is not \(\|r(x)\|\) but its witness-logit projection
\[
e(x)=W(\uphi(x)-\alpha \uT(x)).
\]
The quantity
\[
E_e=\E_x\|e(x)\|^2
\]
is the teacher-path deviation that the witness can see. If \(E_e\) is small relative to the margin already gained by \(\alpha \uT\), then \(L_{\mathrm{TS}}<I_T\) is preserved. In particular, a sufficient tube condition is
\[
Q_\alpha\sqrt{E_e}+\frac14E_e<\Delta_\alpha,
\]
where
\[
\Delta_\alpha=I_T-L_{\mathrm{TS}}(\alpha \uT),
\qquad
Q_\alpha=\sqrt{\E_x\|T_x-S(\alpha \uT(x))\|^2}.
\]
This relation gives a non-vacuous training target: keep the witness-visible deviation from the teacher path below the safety tube determined by \(\Delta_\alpha\).

\subsubsection{A structural Teacher-path guarantee with a certificate-invariant residual route}

A direct structural way to keep the witness loss independent of residual code variation is to separate the Teacher-visible route from a witness-null residual route. Let \(N\) be a basis for the null space of \(W=\beta V\), so
\[
WN=0.
\]
Define the deterministic encoder code as
\[
\uphi(x)=\alpha \uT(x)+Nh_\phi(x),
\]
where \(h_\phi(x)\) is an arbitrary learnable encoder function. Then
\[
W\uphi(x)=\alpha W\uT(x),
\]
and hence
\[
S(\uphi(x))=S(\alpha \uT(x)).
\]
Therefore, if \(\alpha\) is chosen so that \(L_{\mathrm{TS}}(\alpha \uT)<I_T\), then the teacher-witness certificate is guaranteed regardless of the behavior of \(h_\phi(x)\). The role of \(h_\phi(x)\) is to provide a witness-invisible route that remains available to the decoder for factors not represented by the Teacher view.  Certificate invariance alone does not establish which factors the decoder actually places there.

This construction is useful when the latent dimension exceeds the witness rank. If the witness uses \(K\) simplex components, the teacher-visible row space has dimension at most \(K-1\). When \(d_z>K-1\), the null space can carry additional code variation without changing the witness output. Choosing \(\uT(x)\) as the minimum-norm teacher deterministic code in the row space of \(W\) makes \(\uT(x)\) orthogonal to \(Nh_\phi(x)\), so the mean KL cost decomposes into a Teacher-visible cost and a null-route cost; neither term is free with respect to the VAE prior.

\subsubsection{S--R route and decoder usage}
\label{app:decoder-route}

The teacher-path guarantee above certifies a \(T\to S\) route: the teacher view enters the latent witness. It does not by itself prove that the decoder uses that route. A powerful decoder might reconstruct primarily through \(Nh_\phi(x)\) and ignore the protected teacher block. To test decoder usage, one can add an \(S\to R\) or \(T\to R\) reconstruction-side diagnostic, or a swap-based intervention.

Let \(\hat x=D_\theta(z)\), with \(z\sim q_\phi(z\mid x)\), and define the reconstruction-side posterior
\[
R_{\hat x}=T_{\omega^\star}(c\mid \hat x),
\]
obtained by applying the frozen Teacher to the reconstruction. The two optional direct diagnostics are
\[
L_{\mathrm{SR}}
=
\E_{x,z}\KL\left(S(\uphi(x))\|R_{D_\theta(z)}\right),
\qquad z\sim q_\phi(z\mid x),
\]
and
\[
L_{\mathrm{TR}}
=
\E_{x,z}\KL\left(T_x\|R_{D_\theta(z)}\right),
\qquad z\sim q_\phi(z\mid x).
\]
Here \(L_{\mathrm{TR}}\) tests whether the reconstruction preserves the Teacher view, whereas \(L_{\mathrm{SR}}\) tests consistency between the Student witness and the reconstruction-side posterior. These quantities are diagnostics and do not enter the direct certificate \(G_T=I_T-L_{\mathrm{TS}}\).

A stronger intervention-based diagnostic is the cross-swap route. For two samples \(x_i,x_j\), decode
\[
\hat x_{i\leftarrow j}
=
D_\theta\!\left(\alpha \uT(x_j)+Nh_\phi(x_i)\right).
\]
If the decoder uses the Teacher-visible route, the reconstruction-side posterior should follow the swapped Teacher block:
\[
R_{\hat x_{i\leftarrow j}}\approx T_{x_j}.
\]
This gives the intervention loss
\[
L_{\mathrm{TR\text{-}swap}}
=
\E_{i,j}\KL\!\left(
T_{x_j}\|
R_{D_\theta(\alpha \uT(x_j)+Nh_\phi(x_i))}
\right).
\]
Thus \(T\to S\) certifies that the Teacher view enters the latent route, while \(S\to R\), \(T\to R\), and swap-based diagnostics test whether Reconstruction preserves and uses that route. This distinction is essential for decoder bypass and decoder ignorance.

\subsubsection{Practical use}
The simplest procedure chooses \(\alpha\) from the analytic Teacher path, optionally prefits the encoder toward \(\alpha\uT(x)\) in witness-logit space, and then trains \(\mathcal L_{\mathrm{VAE}}+\lambda_{\mathrm{TS}}L_{\mathrm{TS}}\).  A stronger routed extension uses \(\uphi(x)=\alpha\uT(x)+Nh_\phi(x)\) and adds \(S\to R\), \(T\to R\), or swap diagnostics only when decoder use must be tested.  Adaptive weights/scales are possible but are not part of the main empirical method.

\subsection{Derivation of the escape field at constant collapse}
\label{app:escape-gradient-derivation}
\mainback{sec:main-theory}{Section~\ref*{sec:main-theory} (departure/escape result)}
For one sample, let
\[
a_k(x)=\beta v_k^\top u(x)+\log\bar T_k,\qquad S(x)=\softmax(a(x)).
\]
Because
\[
\KL(T_x\|S(x))=\mathrm{const}-T_x^\top a(x)+\log\sum_j e^{a_j(x)},
\]
we have the standard softmax-gradient identity
\[
\nabla_{a(x)}\KL(T_x\|S(x))=S(x)-T_x,
\]
and therefore
\[
\nabla_u\ell_x=\beta\sum_{k=1}^K\bigl(S_k(x)-T_{x,k}\bigr)v_k,
\qquad
-\nabla_u\ell_x=\beta\sum_{k=1}^K\bigl(T_{x,k}-S_k(x)\bigr)v_k.
\]
At an arbitrary constant deterministic code $u(x)\equiv u_0$, write $s_0=S(u_0)$.  The negative functional gradient is
\[
F_T^{u_0}(x)=\beta\sum_{k=1}^K\bigl(T_{x,k}-s_{0,k}\bigr)v_k.
\]
If $F_T^{u_0}(x)=0$ almost everywhere, then affine independence of the simplex vertices, together with $\sum_k(T_{x,k}-s_{0,k})=0$, implies $T_x=s_0$ almost everywhere.  Averaging gives $s_0=\bar T$, hence $I_T=\E\KL(T_X\|\bar T)=0$.  Consequently,
\[
I_T>0\quad\Longrightarrow\quad F_T^{u_0}\not\equiv0
\]
for every exact constant code $u_0$: no such input-independent constant mean is a function-space stationary point of the alignment loss.

At the centered origin, $S(0)=\bar T$, so
\[
F_T(x)=\beta\sum_{k=1}^K(T_{x,k}-\bar T_k)v_k.
\]
The field is sample-dependent even though $\E_XF_T(X)=0$.  Thus a common translation is not the escape direction; a perturbation $u_\varepsilon(x)=u_0+\varepsilon F_T^{u_0}(x)$ has first-order derivative
\[
\left.\frac{d}{d\varepsilon}L_{\rm TS}(u_\varepsilon)\right|_{\varepsilon=0}
=-\E\|F_T^{u_0}(X)\|_2^2<0
\]
whenever $I_T>0$.  This is the function-space departure statement used in the main text; parameter-space realization still depends on the encoder Jacobian and is checked empirically rather than asserted as an optimizer-independent theorem.

\subsection{Gradient dominance and practical variants}
\label{app:gradient-practical}

At the prior-centered collapsed origin $u=0$, the standard-Gaussian mean-KL gradient with respect to the deterministic code is zero. If the reconstruction gradient transverse to the collapse manifold has norm at most $B$, and the RST escape field has norm at least $\kappa$, then choosing $\lambda_{\mathrm{TS}}\kappa>B$ makes the RST transverse component dominant. In practice, $B$ and $\kappa$ can be estimated with batch gradient norms, and possible cancellation can be monitored with gradient angles.

Analytic prefit can train the deterministic encoder code toward $\uT(x)$ or toward the scaled target $\alpha_\star(\tau) \uT(x)$ before full VAE training. A margin-constrained variant can also be used in implementation,
\[
\left[L_{\mathrm{TS}}-(I_T-\tau)\right]_+^2.
\]
This implementation device can maintain a chosen margin, while the certificate theorem itself depends only on $G_T=I_T-L_{\mathrm{TS}}$.

\subsection{Witness weights and escape training}
\label{app:witness-weights}
With several protected views, use
\[
\mathcal L=\mathcal L_{\mathrm{VAE}}+\sum_{a\in\mathcal A}\lambda_a
\E_x\KL\!\left(T_x^a\|S_a(H_a(x))\right).
\]
Separate $\lambda_a$ values accommodate different witness scales or selectively protect mean, variance, block, hierarchy, decoder, or temporal routes.  Multiple witnesses can supply additional transverse directions at a collapsed/low-margin state, but only when their gradients are not cancelled by the VAE objective or by each other; this is not an optimizer-independent escape guarantee.

In practice, fixed weights, warm-up/prefit, gradient-norm balancing, and cosine-angle diagnostics are sufficient controls.  A constrained variant can treat $G_a\ge\tau_a$ as a margin constraint and adapt $\lambda_a$ as a Lagrange multiplier, but the reported experiments deliberately use fixed weights and final margins so that the empirical claim remains about the prescribed code-route certificate.

\FloatBarrier
\section{Extensions, limitations, and re-instantiations}
\label{app:extensions-scope}
\mainback{sec:main-discussion}{Section~\ref*{sec:main-discussion} (scope and limitations)}
This final block separates theorem-level re-instantiations from prospective uses.  Each extension inherits only the route-local certificate under its stated assumptions; empirical status and out-of-scope claims are made explicit.

\subsection{Scope and evidence level of the re-instantiation principle}
\label{app:collapse-taxonomy}
Corollary~\ref{cor:main-reinstantiation} is a statement about a \emph{chosen view/representation pair}, not a claim that one witness solves every phenomenon called ``collapse.''  For a representation/statistic $Z_m$ that satisfies the stated Markov condition, the theorem supplies an exact input-independence boundary and a sufficient information margin.  What it does not supply automatically is an appropriate Teacher for every failure mode, a causal decoder-use test, an aggregate-distribution objective, or empirical validation of every possible route.  Table~\ref{tab:reinstantiation-scope} separates the mathematical interface from the evidence actually provided in this paper.

\begin{table}[H]
\centering
\small
\caption{Scope of the re-instantiation principle.  ``Theorem gives'' refers to the route-local certificate under its assumptions; the last column records the level of evidence in this paper.}
\label{tab:reinstantiation-scope}
\begin{tabular}{p{0.25\linewidth}p{0.35\linewidth}p{0.31\linewidth}}
\toprule
\textbf{Route / view} & \textbf{What the theorem gives} & \textbf{Evidence in this paper} \\
\midrule
Deterministic latent mean / code & Exact input-independence boundary, information certificate, and---with the simplex witness---closed-form realization, routing, path/cost, escape, preservation, and composition. & Main end-to-end setting; multi-dataset checks and five-seed CIFAR-100 studies. \\
Log-variance statistic $\log\sigma_\phi^2(x)$ & Its own route-specific boundary and margin for a fixed uncertainty-related Teacher. & CIFAR-10 cross-statistic sanity check only; not a general variance-collapse study. \\
Prescribed latent block & A block-local boundary and certificate when the witness reads only that block. & Five-seed coarse/fine block-routing and composition study. \\
Grouped Teacher view & Data-processing relation $I_A\le I_T$ plus a group-level route certificate; sufficiently strong fine-view alignment can imply a coarse-view certificate. & Coarse/fine CIFAR-100 instance plus formal derivation. \\
Multi-level hierarchy & Repeated/grouped certificates and chain-rule structure for explicitly defined hierarchy levels. & Formal consequence; the CIFAR-100 coarse/fine study validates the first nontrivial two-level instance, while higher-depth hierarchies are not separately studied. \\
Decoder-side statistic & A route-local certificate can be attached to a fixed decoder-side statistic generated from $X$. & Diagnostic interface only; no intervention study establishing causal decoder use. \\
Temporal / world-model state & The same certificate can be instantiated for a designated state statistic if the conditional-independence assumptions are appropriate. & Prospective interface only; no rollout, dynamics, causal, or long-horizon validation. \\
Aggregate posterior / generation coverage & Not covered by Corollary~\ref{cor:main-reinstantiation} as a simple sample-wise route without defining a different aggregate-level object. & Outside the present theorem and experiments. \\
Blur / perceptual generation quality & No direct consequence of the route-local information certificate. & Outside scope; requires likelihood, decoder, perceptual, or task-specific analysis. \\
\bottomrule
\end{tabular}
\end{table}

The certificate is route-local, while the constructive simplex geometry is developed most fully for the deterministic mean route.  Table~\ref{tab:reinstantiation-scope} records the evidence level for every extension; mode collapse, aggregate-posterior mismatch, causal decoder use, and world-model dynamics remain outside the present claims.

\subsection{Posterior-statistic witnesses for variance collapse}
\label{app:variance-witness}
For a diagonal Gaussian encoder,
\[
q_\phi(z\mid x)=\mathcal N\!\left(\uphi(x),\operatorname{diag}(\sigma_\phi^2(x))\right),
\]
the posterior mean and log-variance are two statistics of the same latent random variable.  The main certificate acts on $\uphi(x)$.  A constant variance by itself is therefore not a collapse of the learned deterministic representation; full exact posterior collapse would require the relevant posterior statistics to become input-independent together.

The route-local principle can nevertheless be re-instantiated on the uncertainty statistic $\ell_\phi(x)=\log\sigma_\phi^2(x)$.  Given a fixed full-support uncertainty Teacher $T_x^\sigma$ with marginal $\bar T^\sigma$, define
\[
S_\sigma(\ell)=\softmax(\beta_\sigma V_\sigma\ell+\log\bar T^\sigma),
\]
\[
I_T^\sigma=\E_x\KL(T_x^\sigma\|\bar T^\sigma),\qquad
L_{T\sigma}=\E_x\KL\!\left(T_x^\sigma\|S_\sigma(\ell_\phi(x))\right),\qquad
G_T^\sigma=I_T^\sigma-L_{T\sigma}.
\]
By Corollary~\ref{cor:main-reinstantiation}, $G_T^\sigma>0$ certifies that the observed log-variance route is not input-independent with respect to that uncertainty view; when the witness has a constant code producing $\bar T^\sigma$, $I_T^\sigma$ is again the attainable input-independent boundary.  The uncertainty Teacher should encode an uncertainty-related view (for example ambiguity, perturbation instability, occlusion, noise level, or prediction uncertainty), rather than being mechanically identified with the semantic code Teacher.

\paragraph{Cross-statistic sanity check.}
We retain the small CIFAR-10 re-instantiation used to check this extension.  A collapsed log-variance statistic gives $G_T^\sigma=0$; a weak input-dependent statistic gives $0.240$; the analytic statistic gives $2.728$; and a noisy analytic statistic gives $2.394$.  This is a route-local sanity check of the same exact boundary, not evidence that the main VAE experiments solve general posterior-variance collapse.

\paragraph{Joint posterior statistics and scope.}
One may likewise define a joint route $h_\phi(x)=[\uphi(x);\rho\ell_\phi(x)]$ and attach a route-specific witness.  A positive joint margin only certifies that the joint observed statistic is not input-independent (an OR-type statement); it does not imply that both channels are individually non-collapsed.  Channel-wise claims require separate margins, e.g. $G_T^\mu>0$ and $G_T^\sigma>0$.  The main experiments therefore keep the primary certificate on the deterministic mean code and use the variance construction only as a targeted re-instantiation of Corollary~\ref{cor:main-reinstantiation}.

\subsection{Structured and temporal re-instantiations}
\label{app:localized-witnesses}

The route-local theorem applies whenever a fixed Teacher view is paired with a designated statistic or route generated from $X$.  Latent-block certificates are the direct structured case: write $H(x)=(H_1(x),\ldots,H_B(x))$ and assign a fixed Teacher $T_x^{(b)}$ and witness $S_b$ to block $H_b$.  Then $G_b>0$ certifies prescribed Teacher information on that block, which is more specific than an active-unit count but remains limited to the assigned witness view.

Decoder use is a separate question; Appendix~\ref{app:nonpositive-teacher-path} defines the $S\to R$, $T\to R$, and cross-swap diagnostics.

For temporal/world-model states, one may choose $Z_m$ to be a designated state block at time $t$ and define fixed Teachers for object, motion, occupancy, interaction, or risk views.  The same $I_m,L_m,G_m$ boundary then tests whether the selected state representation is input-independent with respect to that view.  This is an audit statement only: rollout quality, interventions, causal dynamics, and long-horizon consistency require separate experiments.

\paragraph{Progressive Teacher extraction (prospective).}
Suppose $m$ Teacher views have already been frozen and separately certified.  A new residual or coarse-to-fine view may be fitted and then frozen as $T^{(m+1)}(c\mid x)$.  If it is full-support and nontrivial, its own $(I_{m+1},L_{m+1},G_{m+1})$ inherits the same exact constant-boundary proof.  Progressive extraction can therefore be organized as a sequence of separately auditable certificates.  What is not proved is automatic semantic discovery: defining a meaningful residual, avoiding redundant Teachers, and establishing stability remain separate problems.

\subsection{Distributions of the sampled latent, aggregate posterior, and Teacher code}
\label{app:distributions}

The VAE encoder specifies a conditional Gaussian posterior,
\[
q_\phi(z\mid x)=\N(\uphi(x),\operatorname{diag}\sigma_\phi^2(x)).
\]
Thus $z\mid x$ is Gaussian. The aggregate posterior,
\[
q_\phi(z)=\int q_\phi(z\mid x)p_{\rm data}(x)\,dx,
\]
is generally a mixture of conditional Gaussians and need not equal a single standard Gaussian. This is compatible with the standard VAE setup: the prior is $p(z)=\N(0,I)$, and the ELBO penalizes each conditional posterior through $\KL(q_\phi(z\mid x)\|p(z))$.

The analytic target $\uT(x)$ is the pushforward of the data distribution through the centered log-odds map. It is generally not Gaussian. For label-smoothed teachers it may take finitely many values; for soft GMM teachers it varies continuously with $x$. The certificate is evaluated through the deterministic path $\uphi(x)\mapsto S(\uphi(x))$, because this path is stable and directly tests input dependence of the deterministic encoder code. Sampling quality and aggregate prior matching should be reported separately when unconditional generation is discussed.

\subsection{Group and hierarchical teacher views}
\label{app:group-hierarchy}

The main theorem applies to a single $K$-component teacher. The same constant-baseline argument can be applied to coarser or hierarchical views derived from the same teacher.

Let $A\in\{0,1\}^{K_A\times K}$ be a fixed grouping matrix whose rows define a partition of the $K$ components, so each original component belongs to exactly one group. Define
\[
T_x^A=A T_x,\qquad \bar T^A=A\bar T.
\]
The information content of this group view is
\[
I_A=\E_x\KL(T_x^A\|\bar T^A).
\]
By the data-processing inequality for KL divergence,
\[
0\le I_A\le I_T.
\]
If $I_A=0$, the group view is constant and gives no group-level certificate. If $I_A>0$, the same simplex construction with $K_A$ vertices gives a group-level witness $S^A$ and a group-level alignment loss
\[
L_A=\E_x\KL(T_x^A\|S^A(z_A(x))).
\]
Then
\[
L_A<I_A
\]
rules out input-independent constant collapse of the group-level code $z_A(x)$.

There are two ways to use this observation. First, one may define a separate latent subspace $z_A$ and a separate group witness. At the collapsed point $z_A(x)\equiv0$, the group-level escape force is
\[
F_A(x)=\beta_A\sum_{b=1}^{K_A}\bigl(T_{x,b}^A-\bar T_b^A\bigr)v_b^A,
\]
which is not identically zero whenever $I_A>0$. This is the direct group analogue of
Proposition~\ref{prop:main-frontier}.

Second, if the group prediction is obtained by merging the original student probabilities, $S^A=A S$, then KL contraction gives
\[
L_A=\E_x\KL(A T_x\|A S(\uphi(x)))\le
\E_x\KL(T_x\|S(\uphi(x)))=L_{\mathrm{TS}}.
\]
Thus, a sufficiently strong original margin automatically certifies group views:
\[
L_{\mathrm{TS}}<I_A
\quad\Longrightarrow\quad
L_A<I_A.
\]
Along the analytic $\alpha$-path, a sufficient group-crossing threshold is
\[
\alpha_A^\star
=
\inf\{\alpha\in[0,1]: L_{\mathrm{TS}}(\alpha)< I_A\}.
\]
This sufficient threshold need not equal the first $\alpha$ for which $L_A(\alpha)<I_A$. As $\alpha$ increases, progressively more group views may cross their own constant baselines. This gives a quantitative but limited sense in which partial collapse is improved: the protected teacher information is expressed at more coarse or hierarchical levels. It does not imply that unprotected latent coordinates outside the witness subspace are used.

A hierarchical teacher is obtained by representing the component index as a sequence of decisions,
\[
C=(B_1,B_2,\ldots,B_m).
\]
For a binary hierarchy, $m$ can be of order $\lceil\log_2 K\rceil$. The information decomposes by the chain rule,
\[
I(X;C)
=
I(X;B_1)+I(X;B_2\mid B_1)+\cdots+I(X;B_m\mid B_{<m}).
\]
This gives a structured interpretation of non-collapse: early levels certify coarse teacher information, while later conditional levels certify finer distinctions inside previously selected groups. In this view, partial non-collapse is not a coordinate-wise statement; it is a statement about which teacher views, from coarse to fine, are carried by the latent representation.

\subsection{Fixed GMM responsibility witnesses}
\label{app:gmm-witness}

The simplex witness can be interpreted as a spherical-GMM responsibility head over the latent space. Let
\[
m_k=a v_k,
\qquad
\Sigma_k=\sigma^2I,
\qquad
\pi_k=\bar T_k.
\]
Then
\[
\frac{\pi_k\N(z;m_k,\Sigma_k)}{\sum_j\pi_j\N(z;m_j,\Sigma_j)}
=\softmax_k\left(\frac{a}{\sigma^2}v_k^\top z+\log\bar T_k\right).
\]
Thus the simplex witness is the most symmetric fixed GMM-responsibility witness, with $\beta=a/\sigma^2$. The latent prior is still $\N(0,I)$; the GMM is used only to define the witness map $S(c\mid z)$.

\end{document}